\documentclass{article}
\usepackage{ijcai24}

\usepackage{fontawesome5}
\usepackage{times}
\usepackage{soul}
\usepackage{url}
\usepackage{xspace}
\usepackage[hidelinks]{hyperref}
\usepackage[utf8]{inputenc}
\usepackage[small]{caption}
\usepackage{graphicx}
\usepackage{amsmath}
\usepackage{amsthm}
\usepackage{booktabs}
\usepackage{algorithm}
\usepackage[switch]{lineno}
\usepackage{cleveref}

\usepackage[T1]{fontenc}
\usepackage[utf8]{inputenc}
\usepackage{amssymb}
\usepackage{multicol}

\usepackage{xcolor}
\usepackage{tikz}
\newtheorem{example}{Example}
\newtheorem{theorem}{Theorem}
\newtheorem{definition}{Definition}
\newtheorem{fact}{Fact}
\newtheorem{proposition}{Proposition}
\newtheorem{lemma}{Lemma}
\newtheorem{corollary}{Corollary}

\newcommand\citeappendix{the appendix}

\title{A Logic for Reasoning about Aggregate-Combine Graph Neural Networks}

\author{
Pierre Nunn$^1$\and
Marco Sälzer$^2$\and
François Schwarzentruber$^{1}$\And
Nicolas Troquard$^3$\\
\affiliations
$^1$University of Rennes, IRISA, CNRS, France\\
$^2$Theoretical Computer Science / Formal Methods, University of Kassel, Germany\\
$^3$Gran Sasso Science Institute, L'Aquila, Italy\\
\emails
pierre.nunn@ens-rennes.fr, marco.saelzer@uni-kassel.de, francois.schwarzentruber@ens-rennes.fr, nicolas.troquard@gssi.it
}
\newcommand{\citet}[1]{\citeauthor{#1}~\shortcite{#1}}

\renewcommand{\phi}{\varphi}
\newcommand{\set}[1]{\{#1\}}

\newcommand\setvertices V
\newcommand\setedges E
\newcommand\labeling \ell
\newcommand\setR{\mathbb R}

\newcommand{\aGNN}{\mathcal{A}}
\newcommand{\semanticsof}[1]{[[#1]]}

\newcommand{\statetv}[2]{x_{#1}(#2)}
\newcommand{\statet}[1]{x_{#1}}
\newcommand{\multiset}[1]{\{\{#1\}\}}

\newcommand{\logicKsharp}{K^{\#, -1}}
\newcommand{\logicKsharpone}{K^{\#}}

\newcommand{\Ap}{Ap}
\newcommand{\modalitynumber}{\#}
\newcommand{\istrue}[1]{1_{#1}}

\newcommand\union\cup

\newcommand\sigmabold{\vec{\sigma}}

\newcommand{\lbox}{\square}
\newcommand{\lboxupto}[1]{\lbox^{\rightarrow #1}}
\newcommand{\limply}{\rightarrow}

\newcommand{\co}{\ensuremath{\mathsf{co}}}
\newcommand{\EXPTIME}{\ensuremath{\mathsf{EXPTIME}}\xspace}

\newcommand{\PSPACE}{\ensuremath{\mathsf{PSPACE}}\xspace}

\begin{document}

\maketitle

\begin{abstract}
We propose a modal logic in which counting modalities appear in linear inequalities. We show that each formula can be transformed into an equivalent graph neural 
network (GNN). We also show that a broad class of GNNs can be transformed efficiently into a
formula, thus significantly improving upon the literature about the logical expressiveness of GNNs. We also show that the satisfiability problem is \PSPACE-complete. These results bring together the promise of using standard logical methods for reasoning about GNNs and their properties, particularly in applications such as GNN querying, equivalence checking, etc.
We prove that such natural problems can be solved in polynomial space. \end{abstract}

\newcommand{\Pgeneric}[1]{\textsf{\upshape{P#1}}\xspace}
\newcommand{\Pone}{\Pgeneric{1}}
\newcommand{\Ptwo}{\Pgeneric{2}}
\newcommand{\Pthree}{\Pgeneric{3}}
\newcommand{\Pfour}{\Pgeneric{4}}

\section{Introduction}

Graph Neural Networks (GNNs) perform computations on graphs or on pairs of graphs and vertices, referred to as pointed graphs. 
As a prominent deep learning model, GNNs find applications in various domains such as social recommendations \cite{DBLP:journals/kbs/SalamatLJ21}, drug discovery \cite{XIONG20211382}, material science and chemistry \cite{Reiser22GGNmaterialchemistry}, knowledge graphs \cite{Zi22KG}, among others (see \citet{DBLP:journals/aiopen/ZhouCHZYLWLS20} for an overview).
This growing adoption of GNNs comes with certain challenges. The use of GNNs in safety-critical 
applications sparks a significant demand for safety certifications, given by dependable verification
methods. Furthermore, human understandable explanations for the behaviour of GNNs are needed in order
to build trustworthy applications, coming with the need for a general understanding of the capabilities and 
limitations inherent in specific GNN models. 

In general, there are two approaches enabling rigorous and formal reasoning about GNNs: formal verification \cite{HuangKRSSTWY20} and formal explanation methods \cite{formalXAI} 
for neural network models such as GNNs. 
Formal verification procedures are usually concerned with the sound and complete verification 
of properties like ``Does GNN $\aGNN$ produce an unwanted output $y$ in some specified region $Y$?'' (reachability)
or ``Does GNN  $\aGNN$ behave as expected on inputs from some specified region $X$?'' (robustness). 
Formal explainability methods are concerned with giving answers for questions like 
``Is there a minimal, humanly interpretable reason that GNN  $\aGNN$ produces output $y$ given some input $x$?'' (abductive explanations).  
In both cases, formal verification and formal explanation, logical reasoning offers an all-purpose tool. 
For example, the following algorithm enables addressing correspondence questions:
given some GNN $\aGNN$, produce a logical formula $\phi$ such that $\semanticsof{\aGNN} = \semanticsof{\phi}$ 
where $\semanticsof \aGNN$ is the class of (pointed) graphs recognized by the GNN $\aGNN$, 
and $\semanticsof \phi$ is the class of pointed graphs in 
which $\phi$ holds. Informally put, the goal is to compute a formula $\phi$ that completely characterizes the class 
of (pointed) graphs recognized by GNN $\aGNN$. Given this, one can then investigate GNNs purely based on 
this logical characterization. 
Unfortunately, the synthesis of a formula of a logic, say first-order logic (FO) or modal logic (ML), that captures a 
semantic condition can be notoriously challenging (e.g., \citet{DBLP:conf/aaai/PinchinatRS22}).

\begin{figure}
	\begin{center}

		\begin{tikzpicture}[fill=lightgray, yscale=0.3]
                        \draw (2.5,0) circle (2.25);                        
                        \fill (2.5,0) circle (2.25) (4,0) node {
                        \begin{minipage}{1cm}
                        $\logicKsharpone\\ %\logicKsharp\\
                         \text{GNN}$
                        \end{minipage}};
                        \draw[draw=none] (1.5,0) circle (1.25) (2.25,0) node {GML};
                        \draw (1,0) circle (0.75) node {ML};
			\draw (0,0) ellipse (3 cm and 2.5 cm) (-2,0) node {FO};
		\end{tikzpicture}
	\end{center}
	\caption{Expressivity of our logic $\logicKsharpone$ compared to modal logic (ML), graded modal logic (GML) and first-order logic (FO).\label{figure:expressivity}}
\end{figure}
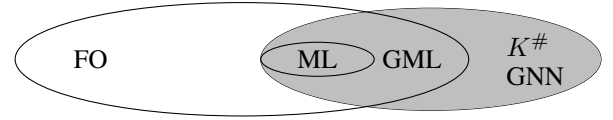

\citet{barcelo_logical_2020} have shown that any formula of Graded Modal Logic (GML) can be transformed into a GNN. Conversely, they have shown that given a GNN \emph{that is expressible in first-order logic} (FO), there exists an equivalent formula in GML. Doing so, they also characterized GML as being the intersection of FO and GNNs. While their result is promising, there is no full logical characterization of what a GNN can express.

That is why we define a logic called $\logicKsharpone$ combining counting modalities and linear programming, 
which is expressive enough to capture a broad and natural class of GNNs. 
As pictured in Figure~\ref{figure:expressivity}, it is more expressive than graded 
modal logic \cite{MDR_noteongraded} and, thus, $\logicKsharpone$ captures a broader class of GNNs than 
previously identified in \citet{barcelo_logical_2020}. Furthermore, we show that the satisfiability problem of 
$\logicKsharpone$ is PSPACE-complete, leading to immediate complexity results for various formal verification
and explainability problems as reachability, robustness, or producing abductive explanations.

\paragraph{Overview of the Main Contributions.}
\label{section:overview}
We present the logic $\logicKsharpone$ which captures Aggregate-Combine Graph Neural Networks,
also called Message Passing Neural Networks \cite{barcelo_logical_2020,GilmerSRVD17},
where the aggregation function is the sum, the combination and final classification functions are linear functions with integer parameters, and truncated ReLU ($max(0, min(1, x))$) is used for the activation function.
We refer to these GNNs simply as GNNs in the paper.
In particular, we show that:
\begin{itemize}
	\item for each formula $\varphi$ of $\logicKsharpone$ there is a GNN $\aGNN$	recognizing exactly the same pointed graphs as $\varphi$ (Theorem~\ref{thm:fromLogictoGNN}),
	\item for each GNN $\aGNN$ there is a formula $\varphi$ of $\logicKsharpone$ 
	recognizing exactly the same pointed graphs as $\aGNN$ (Theorem~\ref{thm:fromGNNtoLogic}).
\end{itemize}
These results significantly extend the class of GNNs for which a logical characterization is known from \citet{barcelo_logical_2020}.
Furthermore, we provide an algorithm for the satisfiability problem of $\logicKsharpone$, proving that the problem is 
PSPACE-complete (Theorem~\ref{theorem:pspacecompletenessKsharpone}). 
This also provides algorithmic solutions for the following problems:
given a GNN $\aGNN$ and a logical formula $\phi$,  decide whether (\Pone) $\semanticsof{\aGNN} = 
\semanticsof{\phi}$, (\Ptwo) $\semanticsof{\aGNN} \subseteq \semanticsof{\phi}$, (\Pthree) $\semanticsof{\phi} \subseteq 
\semanticsof{\aGNN}$, (\Pfour) $\semanticsof{\phi} \cap \semanticsof{\aGNN} \neq \emptyset$.

\begin{example}\label{ex-music-tuba}
	Consider a setting where GNNs are used to classify users in a social network.
	Assume we have a GNN $\aGNN$ which is \emph{intended} to recommend exactly those users who have 
	at least one friend who is a musician and at most one-third of their friends play the tuba. 
	We call this the \emph{few-tubas} property.
	Given our translation from $\aGNN$ to a $\logicKsharpone$ formula, we can answer questions like 
	\begin{description}
		\item[A]\label{item:A} ``Is each recommended person a person that has the few-tubas property?''
		\item[B] ``Does any recommended person not have the few-tubas property?''
		\item[C] ``Is any person that befriends a musician but also too many tuba players recommended?''
		\item[D] ``Is it possible to recommend a person that does not have the few-tubas property?''
	\end{description}
	by representing the corresponding properties as $\logicKsharpone$ formulas and then solving problem \Pone 
	for \emph{\textbf{A}}, \Ptwo for \emph{\textbf{B}}, \Pthree for \emph{\textbf{C}} and \Pfour for \emph{\textbf{D}}. Question \emph{\textbf{A}} corresponds to giving an abductive explanation, 
	as a positive answer would indicate that ``at least one musician friend and not too many tuba playing friends''
	is a (minimal) reason for 
	``recommended person''. In the same manner, \emph{\textbf{B}} and \emph{\textbf{C}} are reachability properties and \emph{\textbf{D}} is a robustness 
	property.
\end{example}

\paragraph{Outline.}
In Section~\ref{section:background} we recall the necessary preliminaries on graph neural networks.
In Section~\ref{section:logic}, we define the logic $\logicKsharpone$. In Section~\ref{section:correspondence}, we study the correspondence between GNNs and $\logicKsharpone$. In Section~\ref{section:complexitylogic}, we discuss the satisfiability problem of $\logicKsharpone$.  Section~\ref{section:reasoningaboutGNN} addresses the complexity of the problems \Pone--\Pfour. Section~\ref{section:relatedwork} is about the related work, and Section~\ref{section:conclusion} concludes.

\section{Background on GNNs}
\label{section:background}
\label{section:GNN}

\newcommand{\dimensionstate}{d}
\newcommand{\nblayers}{L}
\newcommand{\layer}{\mathcal L}
\newcommand{\AGG}{\mathit{agg}}
\newcommand{\COMB}{\mathit{comb}}
\newcommand{\CLS}{\mathit{cls}}

\tikzstyle{vertex} = [circle,draw, inner sep=0mm,minimum height=3mm,font=\tiny]
\tikzstyle{processarrow} = [line width = 1mm, -latex]

\newcommand{\tikzexamplegraphm}[5]
{
	\begin{tikzpicture}[xscale=1, yscale=0.8]
		\node[vertex, #2] (u) at (-1, 0) {};
		\node[vertex,  fill=blue!20!orange] (v) at (-0.5, 1) {#3};
		\node at (-0.5, 1.4) {$u$};
		\node[vertex,  #4] (w) at (0, 0) {};
		\node[vertex] (y) at (-1, -1) {#5};
		\node at (0, -1) {$#1$};
		\draw[->] (v) edge (u);
		\draw[->] (v) edge (w);
		\draw[->] (u) edge (y);
	\end{tikzpicture}
}

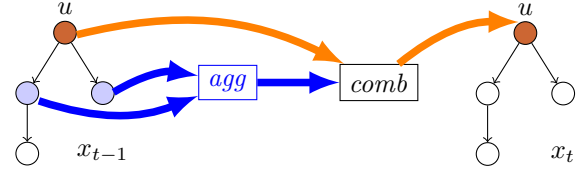
\begin{figure}
	\centering
	\begin{tikzpicture}[node distance=20mm]
		\node (gs) at (-0.35, 0.2) {\tikzexamplegraphm{x_{t-1}}{fill=blue!20!white}{}{fill=blue!20!white}{}};
		\node[draw, right of=gs, blue] (AGG) {$\AGG$};
		\draw[processarrow, blue] (u) edge[bend right=20] (AGG);
		\draw[processarrow, blue] (w) edge[bend left=20] (AGG);
		\node[draw, right of=AGG] (COMB) {$\COMB$};
		\draw[processarrow, blue] (AGG) edge[bend left=0] (COMB);
		\draw[processarrow, orange] (v) edge[bend left=20] (COMB);
		\node[right of=COMB] (g1) {\tikzexamplegraphm{x_t}{}{}{}{}};
		\draw[processarrow, orange] (COMB) edge[bend left=20] (5.5, 1);
	\end{tikzpicture}
	\caption{A layer in a GNN transforms the state $x_{t-1}$ at step $t-1$ into the state $x_t$ at time $t$. The figure shows how $x_t(u)$ is computed. First, the function $\AGG$ is applied to the features in the successors of $u$. Then $\COMB$ is applied to that result and $x_{t-1}(u)$ to obtain $x_t(u)$.}
	\label{figure:layer}
\end{figure}

In this paper, we consider Aggregate-Combine GNNs (AC-GNN) \cite{barcelo_logical_2020}, also sometimes called message passing neural
networks \cite{GilmerSRVD17}. 

An \emph{aggregation function} $\AGG$ is a function mapping finite multisets of vectors in $\mathbb{R}^{\dimensionstate}$ to vectors in $\mathbb{R}^{\dimensionstate}$.
A \emph{combination function} $\COMB$ is a function mapping a vector in $\mathbb{R}^{2\dimensionstate}$ to vectors in $\mathbb{R}^{d'}$.
An \emph{AC-GNN layer} $\layer$ of input dimension $d$ is defined by an aggregation function $\AGG$ and a combination function $\COMB$ of matching dimensions, meaning $\AGG$ expects and produces vectors from $\mathbb{R}^d$ and $\COMB$ has input dimensionality $2d$. 
The output dimension of $\layer$ is given by the output dimension of $\COMB$.
Then, an \emph{AC-GNN} is a tuple $(\layer^{(1)}, \dotsc, \layer^{(\nblayers)}, \CLS)$ where $\layer^{(1)}, \dotsc, \layer^{(\nblayers)}$ are 
$\nblayers$ AC-GNN layers and $\CLS : \setR^\dimensionstate \rightarrow \set{0, 1}$ is a \emph{classification function}. We assume that all GNNs are well-formed in the sense that
output dimension of layer $\layer^{(i)}$ matches input dimension of layer $\layer^{(i+1)}$ as well
as output dimension of $\layer^{(L)}$ matches input dimension of $\CLS$.

\begin{definition}\label{def:our-GNNs}
Here we call a \emph{GNN} an AC-GNN $\aGNN$ where all aggregation functions are given by $\AGG (X) = \sum_{x \in X} x$, 
all combination functions are given by $\COMB(x,y) = \sigmabold(xC+yA +b)$ where $\sigmabold(x)$ is the componentwise application of the \emph{truncated ReLU} $\sigma(x) = max(0, min(1, x))$, where $C$ and $A$ are matrices of integer parameters and $b$ is a vector of integer parameters, and where the classification function is $\CLS(x) = \sum_i a_i x_i \geq 1$, with $a_i$ integers.
\end{definition}

A \emph{(labeled directed) graph} $G$ is a tuple $(\setvertices, \setedges, \labeling)$ such that $\setvertices$ is a finite set of vertices, $\setedges \subseteq \setvertices \times \setvertices$ a set of directed edges and $\labeling$ is a mapping from~$\setvertices$ to a valuation over a set of atomic propositions. We write  $\labeling(u)(p) = 1$ when atomic proposition $p$ is true in $u$, and $\labeling(u)(p) = 0$ otherwise. Given a graph $G$ and vertex $u \in \setvertices$, we call $(G,u)$ a \emph{pointed graph}.

Let $G = (\setvertices, \setedges, \labeling)$ be a graph. A \emph{state} $x$ is a mapping from $\setvertices$ into $\setR^\dimensionstate$ for some $\dimensionstate$. 
When applied to a graph $G$, the $t$-th GNN layer $\layer^{(t)}$ transforms the previous state $\statet {t-1}$ into the next state  $\statet {t}$.
Supposing that the atomic propositions occurring in $G$ are $p_1, \dotsc, p_k$, the \emph{initial state} $x_0$ is defined by: $$x_0(u) := (\labeling(u)(p_1), \dots, \labeling(u)(p_k)) \in \setR^d$$ for all $u \in V$.
Then:
$$\statetv{t}u := \COMB (\statetv {t-1} u,\AGG(\multiset{\statetv {t-1} v | uv \in \setedges}))$$
\noindent
where $\AGG$ and $\COMB$ are respectively the aggregation and combination function of the $t$-th layer. 
In the above equation, note that the argument of $\AGG$ is the multiset of the feature vectors of the successors of $v$. Thus, the same vector may occur several times in that multiset. Figure~\ref{figure:layer} explains how a layer works at each vertex.

\begin{example}
	Consider a layer defined by $\AGG (X) := \sum_{x \in X} x$ and 	
	$$\COMB((x, x'),(y, y')) := \left(\begin{matrix}
		\sigma(x + 2x' + 3y + 4y' + 5) \\
		\sigma(6x + 7x' + 8y + 9y' + 10) \\
	\end{matrix}\right).$$
Suppose as in Figure~\ref{figure:layer},
 that vertex $u$ has two successors named $v$ and $w$. 
 Here we suppose that feature vectors are of dimension~2: 
 $x_{t-1}(u), x_{t-1}(v), x_{t-1}(w) \in \setR^2$.
 First, $\AGG (\set{\set{x_{t-1}(v), x_{t-1}(w)}}) = x_{t-1}(v) + x_{t-1}(w)$. Second, with our example of combination function we get:
 $$x_t(u) := \left(\begin{matrix}
 	\sigma(x_{t-1}(u)_1 + 2x_{t-1}(u)_2 + 3y_1 + 4y_2 + 5) \\
 	\sigma(6x_{t-1}(u)_1 + 7x_{t-1}(u)_2 + 8y_1 + 9y_2 + 10) \\
 \end{matrix}\right) $$
where $y_1 =  x_{t-1}(v)_1 + x_{t-1}(w)_1$ and
$y_2 =  x_{t-1}(v)_2 + x_{t-1}(w)_2$.
\end{example}

Figure~\ref{figure:gnn} explains how a GNN works overall: the state is updated at each layer; at the end the function $\CLS$ says whether each vertex is recognized (output 1) or not (output 0).

\newcommand{\tikzexamplegraph}[5]
{
	\begin{tikzpicture}[scale=0.5]
		\pgfmathparse{70*rnd+30}
		\edef\tmp{\pgfmathresult}
		\node[vertex, fill=white!\tmp!black] (u) at (-1, 4) {#2};
		\pgfmathparse{70*rnd+30}
		\edef\tmp{\pgfmathresult}
		\node[vertex,fill=white!\tmp!black] (v) at (-0.5, 5) {#3};
		\pgfmathparse{70*rnd+30}
		\edef\tmp{\pgfmathresult}
		\node[vertex,fill=white!\tmp!black] (w) at (0, 4) {#4};
		\pgfmathparse{70*rnd+30}
		\edef\tmp{\pgfmathresult}
		\node[vertex,fill=white!\tmp!black] (y) at (-1, 3) {#5};
		\node at (0, 3) {$#1$};
		\draw[->] (v) edge (u);
		\draw[->] (v) edge (w);
		\draw[->] (u) edge (y);
	\end{tikzpicture}
}

\newcommand{\tikzexamplegraphoutput}[5]
{
	\begin{tikzpicture}[scale=0.5]
		\node[vertex] (u) at (-1, 4) {#2};
		\node[vertex] (v) at (-0.5, 5) {#3};
		\node[vertex] (w) at (0, 4) {#4};
		\node[vertex] (y) at (-1, 3) {#5};
		\node at (0, 3) {$#1$};
		\draw[->] (v) edge (u);
		\draw[->] (v) edge (w);
		\draw[->] (u) edge (y);
	\end{tikzpicture}
}

\newcommand{\tikzlayer}[1]{layer #1}

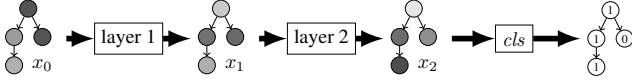
\begin{figure}
	\centering
	\scalebox{0.75}{
	\begin{tikzpicture}[node distance=17mm]
		\node (g0) {\tikzexamplegraph{x_0}{}{}{}{}};
		\node[draw, right of=g0] (l1) {\tikzlayer{1}};
		\node[right of=l1]  (g1) {\tikzexamplegraph{x_1}{}{}{}{}};
		\node[draw, right of=g1] (l2) {\tikzlayer{2}};
		\node[right of=l2]  (g2) {\tikzexamplegraph{x_2}{}{}{}{}};
		\node[draw, right of=g2] (l3) {$\CLS$};
		\node[right of=l3]  (g3) {\tikzexamplegraphoutput{}{1}{1}{0}{1}};
		\draw[processarrow] (g0) -- (l1);
		\draw[processarrow] (l1) -- (g1);
		\draw[processarrow] (g1) -- (l2);
		\draw[processarrow] (l2) -- (g2);
		\draw[processarrow] (g2) -- (l3);
		\draw[processarrow] (l3) -- (g3);
	\end{tikzpicture}}
	\caption{General idea of a GNN with 2 layers applied on a graph with 4 vertices.}
	\label{figure:gnn}
\end{figure}

Let $\aGNN = (\layer^{(1)},\ldots, \layer^{(\nblayers)}, \CLS)$ be a GNN. We define the semantics $\semanticsof \aGNN$ of $\aGNN$ as the set of pointed graphs $(G, u)$ 
such that $\CLS(x_\nblayers(u)) = 1$.

\section{The Logic $\logicKsharpone$}
\label{section:logic}

In this section, we describe the syntax and semantics of $\logicKsharpone$ and its fragment $\logicKsharp$. 

\subsection{Syntax}

\newcommand{\NTexpression}{\xi}

Consider a countable set $\Ap$ of propositions. We define the language of logic $\logicKsharpone$ as the set of formulas generated by the following BNF:
\begin{align*}
	\phi & ::= p \mid \lnot \phi \mid \phi \lor \phi \mid \NTexpression \geq 0 \\ % \mid \NTexpression = 0 \\
	\NTexpression & ::= c \mid \istrue\phi \mid \modalitynumber \phi \mid \NTexpression + \NTexpression \mid c\times \NTexpression 
\end{align*}
where $p$ ranges over $\Ap$, and $c$ ranges over $\mathbb Z$.
Atomic formulas are propositions $p$, inequalities and equalities of linear expressions. We consider linear expressions over $\istrue\phi$ and $\modalitynumber \phi$. The number $\istrue\phi$ is equal to $1$ if $\phi$ holds in the current world and equal $0$ otherwise. The number $\modalitynumber \phi$ is the number of successors in which $\phi$ hold. The language seems strict but we write $\NTexpression_1 \leq \NTexpression_2$ for $\NTexpression_2 - \NTexpression_1 \geq 0$, $\NTexpression = 0$ for $(\NTexpression \geq 0) \land (-\NTexpression \geq 0)$, etc.
Recall that modal logic itself extends propositional logic with a modal construction $\lbox \phi$ whose semantics is `the formula holds in all successors'. Logic $\logicKsharpone$ is an extension of modal logic since we can define $\lbox \phi :=( \modalitynumber (\lnot \phi) \leq 0)$: the number of successors in which $\phi$ does not hold equals $0$. We write $\top$ for a tautology.

\begin{example}
The few-tuba property of \Cref{ex-music-tuba}, users (nodes) who have at least one friend who is a musician and at most one-third of their friends play the tuba, can be represented by the formula
	$(\modalitynumber musician \geq 1) \land (\modalitynumber \top \geq 3 \times \modalitynumber tubaplayer)$.
\end{example}

\newcommand{\subformulasof}[1]{sub(#1)}
The set of subformulas, $\subformulasof{\phi}$ is defined by induction on~$\phi$:
$\subformulasof p  = \set{p}$, 
$\subformulasof {\lnot \phi}  = \set{\lnot \phi} \union \subformulasof{\phi}$,
$\subformulasof{\phi \lor \psi} = \set{\phi \lor \psi} \union \subformulasof \phi \union \subformulasof \psi$, and
$\subformulasof{\NTexpression \geq 0} = \set{\NTexpression \geq 0} \union \bigcup \set{\subformulasof\psi \mid \text{$1_\psi$ or $\modalitynumber \psi$ in $\NTexpression$}}$.

\newcommand{\modaldepthof}[1]{md(#1)}
The modal depth of a formula, $\modaldepthof{\phi}$ and the modal depth of an expression, $\modaldepthof{\NTexpression}$ are defined by mutual induction on $\phi$ and $\NTexpression$:
$\modaldepthof p = \modaldepthof{c}  = 0$,
$\modaldepthof {\lnot \phi} = \modaldepthof {1_\phi} = \modaldepthof{\phi}$,
$\modaldepthof{\NTexpression \geq 0} = \modaldepthof {k\cdot \NTexpression}   = \modaldepthof{\NTexpression}$,
$\modaldepthof {\modalitynumber \phi} = \modaldepthof{\phi} + 1$,
and
$\modaldepthof {\NTexpression_1 + \NTexpression_2} = \max(\modaldepthof {\NTexpression_1},\modaldepthof {\NTexpression_2})$.
As in modal logic, modalities are organized in levels.
\begin{example}
	$\modaldepthof{1_{p \land \modalitynumber {q} \leq 4} \leq \modalitynumber{((\modalitynumber p \geq 2}) \leq 4)} = 2$.
	The expressions $\modalitynumber {q}$ and $\modalitynumber{(\modalitynumber p \geq 2})$ are at the root level (level~$1$), while the expression $\modalitynumber p$ is at level~$2$.
\end{example}

In this paper, a formula is represented by a DAG (directed acyclic graph) instead of just a syntactic tree. DAGs, contrary to syntactic trees, allow the reuse of subformulas.

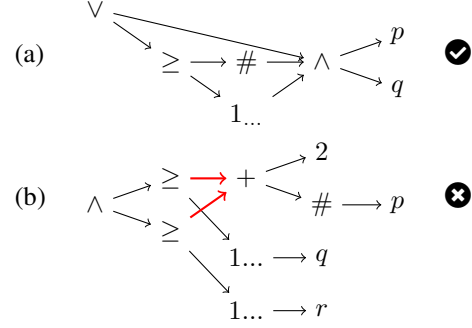
\begin{figure}
	\begin{center}
			\begin{tabular}{lll}
				(a) & 
		\begin{tikzpicture}[yscale=0.7, baseline=0mm]
                \node (lor) at (-1,1) {$\lor$};                                                
			\node (geq) at (0, 0) {$\geq$};
			\node (modalitynumber) at (1, 0) {$\modalitynumber$};
			\node (1) at (1, -1) {$1_{...}$};
			\node (land) at (2, 0) {$\land$};
			\node (p) at (3, 0.5) {$p$};
			\node (q) at (3, -0.5) {$q$};
			\draw[->] (geq) -- (modalitynumber);
			\draw[->] (geq) -- (1);
			\draw[->] (modalitynumber) -- (land);
			\draw[->] (1) -- (land);
			\draw[->] (land) -- (p);
			\draw[->] (land) -- (q);

                \draw[->] (lor) -- (land);
                \draw[->] (lor) -- (geq);
		\end{tikzpicture}
                & \faCheckCircle  \\
	(b) & 
	\begin{tikzpicture}[yscale=0.7, baseline=0mm]
		\node (1) at (0, 0) {$\land$};
		\node (2) at (1, 0.5) {$\geq$};
		\node (3) at (1, -0.5) {$\geq$};
		\node (4) at (2, 0.5) {$+$};
		\node (5) at (3, 1) {$2$};
		\node (6) at (3, 0) {$\#$};
		\node (7) at (4, 0) {$p$};
		\draw[->] (4) -- (6);
		\draw[->] (4) -- (5);
		\node (8) at (2, -1) {$1...$};
		\node (9) at (2, -2) {$1...$};
		\node (10) at (3, -1) {$q$};
		\node (11) at (3, -2) {$r$};
		\draw[->] (2) edge (8);
		\draw[->] (3) edge (9);
		\draw[->] (8) edge (10);
		\draw[->] (9) edge (11);
		\draw[->] (6) edge (7);
		\draw[->] (1) -- (2);
		\draw[->] (1) -- (3);
		\draw[->, red, thick] (2) -- (4);
		\draw[->, red, thick] (3) -- (4);
	\end{tikzpicture}  & \faTimesCircle 

	\end{tabular}
\end{center}
	\caption{DAG representation of formulas. (a)~We allow for reusing subformulas. (b)~We disallow for reusing arithmetical expressions.\label{figure-dag}}
\end{figure}

\begin{example}
The DAG depicted in Figure~\ref{figure-dag}(a), in which the subformula $p \land q$ is used thrice, represents the formula
$(p \land q) \lor (\modalitynumber (p \land q) \geq 1_{p \land q})$.
\end{example}
\noindent However, we disallow DAGs reusing arithmetic expressions. 
\begin{example}
The formula $(2+\modalitynumber p \geq 1_q) \land (2+\modalitynumber p \geq 1_r)$ \emph{cannot} be represented by the DAG shown in Figure~\ref{figure-dag}(b) which refers to the expression $2+\modalitynumber p$ twice.
\end{example}
Formally, the fact that arithmetic expressions are not reused in DAGs representing formulas is reflected by the simple property that the nodes representing arithmetical expressions have in-degree~$1$.
\begin{definition}
	A DAG of a formula is a graph in which nodes for $c$, $1_{\phi}$, $\modalitynumber \phi$, $\NTexpression + \NTexpression'$, $c \times \NTexpression$ have in-degree $1$.
\end{definition}

The reason for not allowing reusing arithmetical expressions in the DAG representation of formulas is technical. Firstly, as demonstrated by Theorem~\ref{thm:fromGNNtoLogic}, the transformation of GNNs into formulas does not necessitate the reusing of arithmetical subexpressions but only subformulas. Second, we will need to efficiently transform formulas in DAG representation into formulas in tree representation (\Cref{lemma-satKone-dag-tree}). For this result, we need DAGs wherein only subformulas are reused, excluding the reuse of arithmetical expressions.
The size $|\phi|$ of a $\logicKsharpone$ formula $\phi$ in DAG form is the number of bits needed to represent the DAG.

The logic $\logicKsharp$ is the syntactic fragment of $\logicKsharpone$ in which constructions $1_\phi$ are disallowed.

\subsection{Semantics}

\newcommand{\semanticsvalue}[2]{[[#1]]_{#2}}

As in modal logic, a formula $\phi$ is evaluated in a pointed graph $(G, u)$ (also known as pointed Kripke model). 
We define the truth conditions $(G,u) \models \phi$ ($\phi$ is true in $u$) by 
	\begin{center}
		\begin{tabular}{lll}
			$(G,u) \models p$ & if & $\labeling(u)(p) = 1$, \\
			$(G,u) \models \neg \phi$ & if & it is not the case that $(G,u) \models \phi$, \\
			$(G,u) \models \phi \land \psi$ & if & $(G,u) \models \phi$ and $(G,u) \models \psi$, \\
			$(G,u) \models \NTexpression \geq 0$ & if &  $\semanticsvalue{\NTexpression}{G,u} \geq 0$, \\
			%$(G,u) \models \NTexpression = 0$ & if &  $\semanticsvalue{\NTexpression}{G,u} = 0$
		\end{tabular}
	\end{center}
	and the semantics $\semanticsvalue{\NTexpression}{G,u}$ (the value of $\NTexpression$ in $u$) of an expression $\NTexpression$ by mutual induction on $\phi$ and $\NTexpression$ as follows.
	\begin{center}
		$\begin{array}{ll}
			\semanticsvalue{c}{G, u} & = c, \\
			\semanticsvalue{\NTexpression_1+\NTexpression_2}{G, u} & = \semanticsvalue{\NTexpression_1}{G,u}+\semanticsvalue{\NTexpression_2}{G,u}, \\
			\semanticsvalue{c \times \NTexpression}{G, u} & = c \times \semanticsvalue{\NTexpression}{G,u}, \\
			\semanticsvalue{\istrue\phi}{G, u} & = \begin{cases}
				1 & \text{if $(G,u) \models \phi$} \\
				0 & \text{else},
			\end{cases} \\  
			\semanticsvalue{\modalitynumber\phi}{G, u} & = |\{v \in \setvertices \mid (u,v) \in \setedges \text{ and } (G,v) \models \phi\}|.
		\end{array}$
	\end{center}
We illustrate it in the next example.
\begin{example}
	\begin{figure}
		\centering
		\begin{tikzpicture}[scale=1, rotate=90]
			\node[vertex] (u) at (-1, 0) {};
			\node[vertex] (v) at (-0.5, 1) {$p$};
			\node at (-0.5, 1.3) {$u$};
			\node[vertex] (w) at (0, 0) {$q$};
			\node[vertex] (y) at (-1, -1) {$p$};
			\draw[->] (v) edge (u);
			\draw[->] (v) edge (w);
			\draw[->] (u) edge (y);
		\end{tikzpicture}
		\caption{Example of a pointed graph $G, u$. We indicate true propositional variables at each vertex.}
		\label{fig:pointedgraph}
	\end{figure}
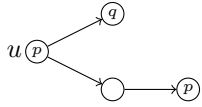
	
	Consider the pointed graph $G, u$ shown in Figure~\ref{fig:pointedgraph}. We have $G, u \models p \land (\modalitynumber \lnot p \geq 2) \land \modalitynumber (\modalitynumber p \geq 1) \leq 1$. Indeed, $p$ holds in $u$, $u$ has (at least) two successors in which $\lnot p$ holds. Moreover, there is (at most) one successor which has at least one $p$-successor.
\end{example}
We define $\semanticsof \phi$ as the set of the pointed graphs $G, u$ such that $G, u \models \phi$.
Furthermore, we say that $\phi$ is \emph{satisfiable} when there exists a pointed graph $G, u$ such that $G, u \models \phi$.

\begin{definition}
	The \emph{satisfiability problem} of $\logicKsharpone$ (resp.\ $\logicKsharp$) is: given a DAG representation of a formula $\phi$ in the language of $\logicKsharpone$ (resp.\ $\logicKsharp$), is $\phi$ satisfiable?
\end{definition}

\subsection{Relationship with Other Logics}

We already mentioned how logic $\logicKsharpone$ extends modal logic.
Logic $\logicKsharpone$ is also an extension of graded modal logic.
Graded modal logic \cite{DBLP:journals/sLogica/Fattorosi-Barnaba85} extends classical modal logic by offering counting modality constructions of the form $\Diamond^{\geq k} \phi$ which means there are at least $k$ successors in which $\phi$ holds.
Logic $\logicKsharpone$ is more expressive than graded modal logic since $\Diamond^{\geq k} \phi$ is rewritten in  $k \leq \modalitynumber \phi$.    
In fact, the expressivity of $\logicKsharpone$ goes beyond FO.

\begin{proposition}\label{example-expressivity-FO}
	There are some properties that can be expressed in $\logicKsharpone$ that cannot be expressed in FO.
\end{proposition}
\begin{proof}
	The property `there are more $p$-successors than $q$-successors' can be expressed in logic $\logicKsharpone$ by the formula $\modalitynumber p \geq \modalitynumber q$, but cannot be expressed in first-order logic, and thus not in graded modal logic.
	This is proven via an Ehrenfeucht-Fra\"iss\'e game (see \citeappendix).
\end{proof}

\section{Correspondence}
\label{section:correspondence}
In this section, we lay the foundations for expressing $\logicKsharpone$-formulas using GNNs,  and vice versa.

\subsection{From Logic to GNNs}

We start by showing that each $\logicKsharpone$-formula is captured by some GNN. The proof follows the same line as 
Prop 4.1 in \citet{barcelo_logical_2020}). However, our result
is a generalization of their result since $\logicKsharpone$ is more expressive than graded modal logic.

\newcommand{\existsC}[1]{\exists^{\geq#1}}

\begin{theorem}\label{thm:fromLogictoGNN}
	For every $\logicKsharpone$-formula $\phi$, 
	we can compute in polynomial time wrt.\ $|\phi|$ a GNN $\aGNN_{\phi}$ such that $\semanticsof \phi= \semanticsof{\aGNN_{\phi}}$.
\end{theorem}

\begin{proof}
	\newcommand*{\geqdepth}{\mathit{geq}}
	Let $\phi$ be a $\logicKsharpone$ formula with occurring propositions $p_1, \dotsc, p_m$. 
	Let $(\phi_1, \dotsc, \phi_n)$ be an enumeration of the subformulas of $\phi$ such that $\phi_i = p_i$ for all 
	$i \leq m$, if $\phi_i \in \subformulasof{\phi_j}$ then $i \leq j$ and $\phi_n = \phi$.
	W.l.o.g.\ we assume that all $\NTexpression \geq 0$ subformulas are of the form $\sum_{j \in J} k_j \times 1_{\phi_j} + \sum_{j' \in J'} k_{j'} \times \modalitynumber{\phi_{j'}} - c \geq 0$
	for some index sets $J, J' \subseteq \{1, \dotsc, n\}$.
	We build $\aGNN_{\phi}$ in a stepwise fashion. Since $\aGNN_{\phi}$ is a
	GNN it is completely defined by specifying the combination function $\COMB_i$ for each layer $l_i$ with $i \in \{1, \dotsc, n\}$ 
	of $\aGNN_{\phi}$ as well as the classification function. 
	
	The output dimensionality of each $\COMB_i$ is $n$. Let $\COMB_1$, namely the first 
	combination in $\aGNN_{\phi}$, have input dimensionality
	$2m$ computing $\COMB_1(x, y) = (x, 0, \dotsc, 0)$.
	Informally, this ensures that state $x_1$ has dimensionality $n$ where the first $m$ dimensions
	correspond to the propositions $p_i$ and all other are $0$.
	Let $\COMB_i(x,y) = \sigmabold(xC + yA + b)$ where $C$, $A$ are $n \times n$ and $b$ is $n$ dimensional and specified as follows.
	All cells of $C$, $A$ and $b$ are zero except for the case $C_{ii} = 1$ if $ i \leq m$, the case $C_{ji} = -1, b_i = 1$ if $\varphi_i = \neg\varphi_j$, 
	the case $C_{ji} = C_{li} = 1$ if $\varphi_i = \varphi_j \lor \varphi_l$ and the case $C_{ji} = k_j, A_{j'i} = k_{j'}, b_i = -c + 1$ for all $j \in J, j' \in J'$ if 
	$\varphi_i = \sum_{j \in J} k_j \times 1_{\phi_j} + \sum_{j' \in J'} k_{j'} \times \modalitynumber{\phi_{j'}} \geq c$. This means that
	all $\COMB_i$ are equal. The classification function $\CLS$ is given by $\CLS(x) = x_n \geq 1$.

The correctness of the construction is proven in \citeappendix.
\end{proof}

	 Furthermore, the number of layers of an equivalent GNN can be reduced to be proportional to $md(\phi)$. In that case, however, the corresponding transformation is a priori not computable in poly-time (see \citeappendix).
	 
\begin{example}
	Consider the formula $\phi = \neg(p \vee (8 \leq 3\times \modalitynumber q))$.
	We define the following GNN $\aGNN$ with 2 layers which is equivalent to $\phi$ as follows. We first consider the following subformulas of $\phi$ in that order: 
	\begin{center}
        \setlength{\tabcolsep}{10pt} % default is 6pt
		\begin{tabular}{ccccc}
                \toprule
			$\phi_1$ & $\phi_2$ & $\phi_3$ & $\phi_4$ & $\phi_5$ \\
			\midrule
			$p$ &  $q$ & $8 \leq 3 \times \modalitynumber q$ &  $\phi_1 \lor \phi_3$ &  $\lnot \phi_4$ \\
                        \bottomrule
		\end{tabular}
        \setlength{\tabcolsep}{6pt} % default
	\end{center}
	
	The combination function for both layers is $\COMB(x,y) = \sigmabold(xC+yA+b) $ where
	\begin{align*}
		C = \begin{pmatrix}
			1 & 0 & 0 & 1 & 0\\
			0 & 1 & 0 & 0 & 0\\
			0 & 0 & 0 & 1 & 0\\
			0 & 0 & 0 & 0 & -1\\
			0 & 0 & 0 & 0 & 0 
		\end{pmatrix},
		&&
		A = \begin{pmatrix}
			0 & 0 & 0 & 0 & 0\\
			0 & 0 & 3 & 0 & 0\\
			0 & 0 & 0 & 0 & 0\\
			0 & 0 & 0 & 0 & 0\\
			0 & 0 & 0 & 0 & 0 
		\end{pmatrix} \\ 
                \textit{and } b = \begin{pmatrix}
			0 & 0 & -7 & 0 & 1\\
		\end{pmatrix}.~~~~
	\end{align*}
	The role of the two first column coordinates are to keep the current values of $p$ and $q$ in 
	each vertex of the graph (hence the identity submatrix in $C$ wrt the first two columns and 
	rows).

	The third column is about subformula $8 \leq 3 \times \modalitynumber q$. The coefficient 3 in $A$ is at the second row (corresponding to $q$).
	 The third column of $xC + yA + b$ is equal to $3 \times \text{number of $q$-successors} - 7$. The third column of $\sigmabold(xC+yA+b) $ equals $1$ exactly when $3 \times \text{number of $q$-successors} - 7 \geq 1$, i.e.\ exactly when $8 \leq 3 \times \modalitynumber q$ should hold.
	The fourth column handles the disjunction of $\phi_1 \lor \phi_3$, hence the $1$ in the first and third rows.
	The fourth column of $\sigmabold(xC+yA+b)$ is $\sigma(x_1 + x_3)$  and equals $1$ iff $x_1 = 1$ or $x_3 = 1$.
	The last column handles the negation $\lnot \phi_4$. The last column of $\sigmabold(xC+yA+b)$ is $\sigma(1 - x_4)$.

\end{example}

\subsection{From GNNs to Logic}

Now, we shift our attention to show how to compute a $\logicKsharpone$-formula that is equivalent to a GNN.
Note that this direction was already tackled by \citet{barcelo_logical_2020} for graded modal logic for the subclass of GNNs that are FO-expressible, 
but their proof is not constructive. Here we give an effective construction. Furthermore, the construction can be done in poly-time in $|\aGNN|$. This point is crucial: it means that we can transform efficiently a GNN into a logical formula and then perform all the reasoning tasks in the logic itself.

\begin{theorem}\label{thm:fromGNNtoLogic}
	Let $\aGNN$ be a GNN. We can compute in polynomial time wrt.\ $|\aGNN|$ a  $\logicKsharpone$-formula $\varphi_\aGNN$, represented as a DAG, such that
	$\semanticsof{\aGNN} = \semanticsof{\varphi_\aGNN}$.
\end{theorem}
\begin{proof}
	Let $\aGNN$ be a GNN with layers $l_1, \dotsc, l_k$ where $\COMB_i$ has input dimensionality 
	$2m_i$, output dimensionality $n_i$ and parameters $C_i$, $A_i$ and $b_i$ and $\CLS(x) = a_1 x_1 + 
	\dotsb + a_{n_k} x_{n_k} \geq 1$. We assume that $m_i =n_{i-1}$ for $i \geq 2$, meaning a well
	defined GNN.
	We build a formula $\varphi_\aGNN$ over propositions $p_1, \dotsc, p_{m_1}$ inductively as follows.
	Consider layer $l_1$. We build formulas $\varphi_{1,j} = \sum_{k=1}^{m_1} 1_{p_k} (C_1)_{kj} + \modalitynumber{p_k} (A_1)_{kj} + b_j \geq 1$ for each $j \in \{1, \dotsc, n_1\}$.
	Now, assume that $\varphi_{i-1,1}, \dotsc, \varphi_{i-1,n_{i-1}}$ are given. Then, we build formulas $\varphi_{i, j} =  \sum_{k=1}^{m_i} 1_{\varphi_{i-1,k}} C_{kj} + 
	\modalitynumber{\varphi_{i-1,k}} A_{kj} + b_j \geq 1$. 
	In the end, we set $\varphi_\aGNN = 1_{\psi}$ where $\psi = a_11_{\varphi_{k,1}}+ \dotsb + 
	a_{n_k}1_{\varphi_{k,n_k}} \geq 1$.

	Let $G,v$ be some pointed graph. The correctness of our construction is straightforward: 
	due to the facts that all parameters in $\aGNN$ are from $\mathbb{Z}$, the value $x_0(u)$ for each $u$ in $G$ 
	is a vector of $0$ and $1$ and that each $\COMB_i$ in $\aGNN$ applies the truncated ReLU pointwise, we have for all $i \in \{1, \dotsc, k\}$ that the 
	value of $x_i(v)$ is again a vector of $0$ and $1$. Therefore, we can capture each $x_i$ using a sequence of $\logicKsharpone$-formulas as done above.
	In combination, we have that $\varphi_\aGNN$ exactly simulates the computation of $\aGNN$.

	The polynomial time of this inductive construction is straightforward: we represent from beginning to
	end the inductively built (sub)formulas in form of a single DAG. This implies that in intermediate steps
	we do not have to rewrite already built subformulas $\varphi_{i-1,k}$, but can simply refer to them. Thus,
	in each step of the inductive procedure we only need $n_i \cdot m_i$ many steps to build all corresponding 
	subformulas. As all $n_i, m_i$ and $k$ are given by $\aGNN$, we have that the procedure is polynomial
	in the size of $\aGNN$ as long as we represent $\varphi_\aGNN$ respectively its subformulas as a DAG. 
\end{proof}

\begin{example}
	The idea is to build formulas that characterize the state $x_t$ (which is the output at layer $t$ if $t \geq 1$, and also the input at layer $t+1$). Initially, $x_0$ is the state whose values comes the truth values of the atomic propositions. We suppose we have only two propositions $p_1, p_2$. Hence, the two formulas that represent state $\statet{0}$ are $\phi_{01} = p_1$ and $\phi_{02} = p_2$. Suppose that the two layers are given by the same combination and aggregation functions: 
	\begin{align*}
		\statetv{1}u= \COMB (\statetv {0} u,\AGG(\multiset{\statetv {0} v | uv \in \setedges})) \\
		\statetv{2}u= \COMB (\statetv {1} u,\AGG(\multiset{\statetv {1} v | uv \in \setedges}))
	\end{align*}
	with aggregation function $\AGG (X) = \sum_{x \in X} x$, and combination function
	$$\COMB((x, x'),(y, y')) = \left(\begin{matrix}
		\sigma(x + 2x' - 3y + 4y' + 5) \\
		\sigma(6x + 7x' + 8y - 9y' + 10) \\
	\end{matrix}\right).$$
	The formulas that represent the state $x_1$ are formula
	$\phi_{11} = 1_{\phi_{01}} + 2\times 1_{\phi_{02}} - 3 \modalitynumber  \phi_{01} + 4  \modalitynumber  \phi_{02} + 5 \geq 1$
	and formula
	$\phi_{12} = 6\times 1_{\phi_{01}} + 7\times 1_{\phi_{02}} + 8 \modalitynumber \phi_{01} - 9  \modalitynumber \phi_{02} + 10 \geq 1$.
	The formulas that represent the state $x_2$ are formula
	$\phi_{21} = 1_{\phi_{11}} + 2\times 1_{\phi_{12}} - 3 \modalitynumber  \phi_{11}+ 4  \modalitynumber  \phi_{12} + 5 \geq 1$ and formula
	$\phi_{22} = 6\times 1_{\phi_{11}} + 7\times 1_{\phi_{12}} + 8 \modalitynumber \phi_{11} - 9  \modalitynumber \phi_{12} + 10 \geq 1$.
	
	Finally, suppose that the classification function is given by $\CLS(x, x') = 5x - 3x' \geq 1$. So the formula $tr(\aGNN)$ that represents the GNN $\aGNN$ is $5 \times 1_{\phi_{21}} - 3 \times 1_{\phi_{22}}  \geq 1$.
\end{example}

\section{Complexity of the Logic}
\label{section:complexitylogic}

In this section, we address the complexity of the satisfiability problem of the logic $\logicKsharpone$. Specifically, we prove that it is \PSPACE-complete (\Cref{theorem:pspacecompletenessKsharpone}). Additional details are available in \citeappendix.

We know from \citet{Ladner77} that the satisfiability problem of the standard modal logic K is \PSPACE-hard, and we observed before that $\logicKsharpone$ is an extension of logic K since we can define $\lbox \phi :=( \modalitynumber (\lnot \phi) \leq 0)$ and we are working with graphs whose relation (represented by the set $\setedges$ of edges) is unconstrained. Hence, we have:
\begin{proposition}\label{prop:pspacehardnesslogicKsharpone}
The satisfiability problem of the logic $\logicKsharpone$ is \PSPACE-hard.
\end{proposition}

\usetikzlibrary{quotes}
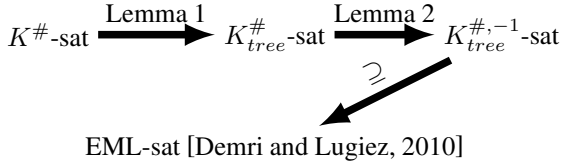
\begin{figure}
\begin{center}  
\begin{tikzpicture}
\node (ks1) at (0, 0) {$\logicKsharpone$-sat};
\node (ks1t) at (3,0) {$\logicKsharpone_{tree}$-sat};
\node (kst) at (6,0) {$\logicKsharp_{tree}$-sat};
\node (dl2010) at (3,-1.5) {EML-sat \cite{DBLP:journals/japll/DemriL10}};

\draw[processarrow] (ks1) edge["\Cref{lemma-satKone-dag-tree}"] (ks1t);
\draw[processarrow] (ks1t) edge["\Cref{lemma-satKoneK}"] (kst);
\draw[processarrow] (kst) to [sloped, "$\supseteq$"] (dl2010);
\end{tikzpicture}
\end{center}
\caption{Schema of the proof to establish the PSPACE upper bound of the satisfiability problem of $\logicKsharpone$. Arrows are poly-time reductions.\label{fig:schema-proof-complexity-logic}}
\end{figure}

To show that the satisfiability problem of $\logicKsharpone$ is also in \PSPACE, we are going to follow a strategy illustrated in \Cref{fig:schema-proof-complexity-logic}:
\begin{enumerate}
\item First we show that the problem can be reduced efficiently to the satisfiability problem of $\logicKsharpone$ with the formulas represented as trees (and not arbitrary DAGs). Let us call this problem $\logicKsharpone_{tree}$-sat. This will be \Cref{lemma-satKone-dag-tree}.
\item Second, we show that $\logicKsharpone_{tree}$-sat can be reduced efficiently to the satisfiability problem of $\logicKsharp$ with formulas represented as trees. This will be \Cref{lemma-satKoneK}.
\item Third, we conclude by simply observing that $\logicKsharp$ can be seen as fragment of Extended Modal Logic (EML) introduced by \citet{DBLP:journals/japll/DemriL10}, whose satisfiability problem is in \PSPACE.
\end{enumerate}
It is sufficient to establish an upper-bound on the complexity of  $\logicKsharpone$, which is stated in the next proposition.
\begin{proposition}\label{prop:pspacemembershipKsharpone}
The satisfiability problem of the logic $\logicKsharpone$ is in \PSPACE.
\end{proposition}

The initial step in both proofs of \Cref{lemma-satKone-dag-tree} and \Cref{lemma-satKoneK} shares the same idea. Given a subformula $\phi$ or arithmetic subexpression $1_\phi$, we introduce a fresh propositional variable $p_\phi$. The variable can then be used as a shortcut, or used \emph{in} a shortcut, to refer to the truth (resp.\ value) of the original subformula (resp.\ arithmetic subexpression).
An adequate formula is then added to `factorize' the original subformula or arithmetic subexpression, and enforced at every `relevant' world of a model to capture the intended properties of the transformations. 
To this end, we can rely on the modality $\lboxupto{m} \phi$, simply defined as
$\lboxupto{m} \phi := \bigwedge_{0\leq i \leq m} \Box^i \phi$,
with $m$ being the modal depth of the original formula.

To prove \Cref{prop:pspacemembershipKsharpone},
we first show that we can efficiently transform a $\logicKsharpone$ formula (represented as a DAG) into an equi-satisfiable $\logicKsharpone$ formula which is represented as a tree.

For every node of the DAG that is a formula, we introduce a fresh proposition.
Starting from the leafs of the DAG, we replace every subformula $\psi$ with its corresponding proposition $p_\psi$ which will simulate the truth of $\psi$.
As we go towards the root, we replace $\psi$ with $p_\psi$, and
we add a formula of the form $\lboxupto{m}(p_\psi \leftrightarrow \psi')$, where $\psi'$ is a syntactic variant of $\psi$, using the previously introduced fresh variables in place of the original formulas.
Since the root is a formula (as opposed to an arithmetic expression), we end with a conjunction of $\lboxupto{m}(p_\psi \leftrightarrow \psi')$ formulas, and a DAG consisting of only one node with the proposition $p_\phi$ which simulates the original formula.
Since we are not duplicating arbitrary formulas, but only propositional variables, this process ensures that the size of the new formula (as a string or as a tree) remains polynomial in the size of the DAG.
The idea is similar to Tseitin transformation \cite{tseitin1983complexity}.
\begin{example}
Consider the DAG depicted on \Cref{figure-dag}(a).
We introduce the propositional variables $p_p$, $p_q$, $p_\land$, $p_\leq$, and $p_\lor$, each corresponding to a node of the DAG denoting a subformula.
The formula is equi-satisfiable with the formula containing the following conjuncts:
$p_\lor$,
$\lboxupto{1}(p_p \leftrightarrow p)$,
$\lboxupto{1}(p_q \leftrightarrow q)$,
$\lboxupto{1}(p_\land \leftrightarrow p_p \land p_q)$,
$\lboxupto{1}(p_\geq \leftrightarrow \modalitynumber p_\land \leq 1_{p_\land})$, and
$\lboxupto{1}(p_\lor \leftrightarrow p_\land \lor p_\geq)$.
For the formula to be true, $p_\lor$ must be true, and so must be $p_\land$ or $p_\geq$, and so on.
\end{example}
The previous example explains the idea of the proof in order to efficiently transform formulas represented as arbitrary DAGs into equi-satisfiable formulas represented as tree. The following lemma states that result as a reduction. 

\begin{lemma}
\label{lemma-satKone-dag-tree}
The satisfiability problem of the logic $\logicKsharpone$ reduces to $\logicKsharpone_{tree}$-sat in poly-time.
\end{lemma}

Now, we show that $1_\chi$ expressions can be removed efficiently and in a satisfiability-preserving way. Technically, $1_\chi$ expressions are replaced by a counting modality, counting successors with a fresh proposition $p_{\chi}$, which artificially simulates the truth of $\chi$.
	\begin{example}
We consider the formula $1_{\chi} \geq 1$. We introduce a fresh propositional variable $p_{\chi}$ and replace the $1_{\chi}$ with a counting modality $\modalitynumber p_{\chi}$. Hence, we simply rewrite the $\logicKsharpone$ formula into the $\logicKsharp$ formula $\modalitynumber p_{\chi} \geq 1$.
We then add a subformula that characterizes the fact that counting $p_{\chi}$-worlds simulates the value of $1_{\chi}$. For that, we say that if $\chi$ holds, then there must be exactly one $p_{\chi}$-successor and if $\chi$ does not hold then there must be no $p_{\chi}$-successor.
At the end, the formula $1_{\chi} \geq 1$ is rewritten into the $\logicKsharp$ formula:
\[
(\modalitynumber p_{\chi} \geq 1) \land \lboxupto{0}((\chi \rightarrow \modalitynumber p_{\chi} = 1) \land (\lnot \chi \rightarrow \modalitynumber p_{\chi} = 0)).
\]
\end{example}
The previous example explains the idea of the proof in order to get rid of $1_\chi$ expressions. The following lemma states that result as a reduction.
\begin{lemma}
\label{lemma-satKoneK}
When formulas are represented as trees, the satisfiability problem of the logic $\logicKsharpone$ reduces to the satisfiability problem of $\logicKsharp$ in poly-time. 
\end{lemma}

The logic $\logicKsharp$ can be seen as a fragment of the logic EML introduced by \citet{DBLP:journals/japll/DemriL10}, where formulas are represented as strings (which have the same size as their syntactic tree representation), and whose satisfiability problem is in \PSPACE. Together with \Cref{lemma-satKone-dag-tree} and \Cref{lemma-satKoneK}, this proves \Cref{prop:pspacemembershipKsharpone}.
\Cref{prop:pspacehardnesslogicKsharpone} and \Cref{prop:pspacemembershipKsharpone} allow us to conclude.
\begin{theorem}\label{theorem:pspacecompletenessKsharpone}
The satisfiability problem of the logic $\logicKsharpone$ is \PSPACE-complete.
\end{theorem}
Since $\co\PSPACE = \PSPACE$, the validity and unsatisfiability problems of $\logicKsharpone$ formulas are also \PSPACE-complete. This will be instrumental in the next section.

\section{Complexity of Reasoning about GNNs}
\label{section:reasoningaboutGNN}

We are now ready to wrap up algorithmic results for reasoning about GNNs.

\begin{corollary}
	\label{corollary:gnnspace}
When considering GNNs,
the problems \Pone--\Pfour are in \PSPACE.
\end{corollary}

\begin{proof}
	Let us prove it for $\Pone$. For the other problems, the principle is similar. Given $\aGNN$, $\phi$, we can check that $\semanticsof{\aGNN} = \semanticsof{\phi}$, by computing the $\logicKsharpone$-formula $tr(\aGNN) \leftrightarrow \phi$ and then check that $tr(\aGNN) \leftrightarrow \phi$ is valid. In order to check for validity, we check that $\lnot (tr(\aGNN) \leftrightarrow \phi)$ is not $\logicKsharpone$-satisfiable. This can be done in poly-space via Theorem~\ref{theorem:pspacecompletenessKsharpone}.\end{proof}

\begin{theorem}
	Problems \Pone--\Pfour are \PSPACE-complete.
\end{theorem}

\begin{proof}
	\newcommand{\aGNNall}{\aGNN_{all}}
	\newcommand{\aGNNnone}{\aGNN_{none}}
	Membership comes from \Cref{corollary:gnnspace}. Hardness follow directly from the \PSPACE-hardness of the satisfiability/unsatisfiability/validity of a $\logicKsharpone$-formula (cf.\ Theorem~\ref{theorem:pspacecompletenessKsharpone}). Consider $\aGNNall$ a GNN that accepts all graphs, and $\aGNNnone$ a GNN that rejects all graphs. Consider the following poly-time reductions:
	\begin{itemize}
\item From validity of $\phi$ to \Pone:
$\semanticsof{\aGNNall} =  \semanticsof{\phi}$.
\item From validity of $\phi$ to \Ptwo:
$\semanticsof{\aGNNall} \subseteq \semanticsof{\phi}$.
\item From unsatisfiability of $\phi$ to \Pthree:
$\semanticsof{\phi} \subseteq  \semanticsof{\aGNNnone}$.
\item From satisfiability of $\phi$ to \Pfour:
$\semanticsof{\phi} \cap \semanticsof{\aGNNall} \neq \emptyset$.
        \end{itemize}
\end{proof}

\section{Related Work}
\label{section:relatedwork}
	The links between graded modal logic \cite{DBLP:journals/sLogica/Fattorosi-Barnaba85} and GNNs have already been observed in the literature 
	\cite{barcelo_logical_2020,grohe_logic_2022}.
We know from \citet{barcelo_logical_2020} that a GNN expressible in first-order logic (FO) is also captured by a formula in graded modal logic.
Nonetheless, graded modal logic has significant limitations as it cannot represent fundamental arithmetic properties essential to GNN computations, especially those that are not expressible in first-order logic (FO).
In a similar vein of finding logical counterparts of GNNs, \citet{DBLP:conf/kr/CucalaGMK23} identify a class of GNNs that corresponds to Datalog.
	Many works combine modal/description logic and quantitative aspects: counting 
	(see, \cite{DBLP:conf/wollic/ArecesHD10,DBLP:journals/japll/DemriL10,DBLP:conf/aiml/Hampson16,DBLP:conf/ecai/BaaderBR20}), or probabilities \cite{DBLP:conf/aaai/ShiraziA07}. \citet{DBLP:conf/kr/GallianiKT23} extend the basic description logic $\mathcal{ALC}$ with a new operator 
	to define concepts by listing features with associated weights
	and by specifying a threshold. They prove that reasoning wrt to a knowledge-base is $\EXPTIME$-complete.
	Linear programming and modal logic have already been combined to solve the satisfiability problem of graded/probabilistic modal logic \cite{DBLP:conf/lpar/SnellPW12}. Our logic $\logicKsharp$ can be seen as a `recursification' of the logic used in  
	\citet{DBLP:journals/corr/abs-2206-05070}. They allow for counting successors satisfying a given feature, and not any subformula. Interestingly, they allow for counting also among all vertices in the graph (sort of counting universal modality). Their logic is proven to be undecidable by reduction from the Post correspondence problem. Contrary to our setting, they use their logic only to characterize labeled graphs, but not to give a back and forth comparison with the GNN machinery itself.
	
	There are different ways to address explanation and verification of GNNs. One approach is to use GNNs, 
	which are explainable by design as considered by \citet{DBLP:journals/corr/abs-2205-13234}. This is of 
	course a deep debate: using models easy to use for learning, versus interpretable models 
	\cite{DBLP:journals/natmi/Rudin19}.  The choice depends on the target application.
	\citet{DBLP:journals/pami/YuanYGJ23} provide a survey on methods used to provide 
	explanations for GNNs by using black-box techniques.
	%According to them, they are instance-level and model-level explanations.
        Instance-level explanations 
	explain why a graph has been recognized by an GNN; model-level ones how a given GNN works. There are also methods based on Logic Explained Networks and variants to generate logical 
	explanation candidates $\phi$ \cite{DBLP:journals/corr/abs-2210-07147}. Once a candidate is generated we 
	could imagine to use our problem \Pone (given in the introduction) to check whether $\semanticsof \aGNN = 
	\semanticsof \phi$, and thus being able to fully synthesize a trustworthy explanation. 
	The logic $\logicKsharpone$ and the results presented in this paper are thus precious tools to assist in the search of model-level explanations of GNNs. 
								
\section{Conclusion and Outlook}
\label{section:conclusion}

	We presented logic $\logicKsharpone$, capturing a broad and natural class of GNNs. Furthermore, we 
	proved that the satisfiability problem of $\logicKsharpone$ is \PSPACE-complete, leading to
	direct practical perspectives to solve formal verification and explanation problems regarding this class of GNNs.
	
	There are several directions to go from here. First, we aim to consider a larger class of GNNs. This will 
	require to augment the expressivity of the logic, for instance by adding other activation functions like
	ReLU in an extension of $\logicKsharpone$. Fortunately, SMT solvers have been extended to capture ReLU 
	(for example see \citet{DBLP:conf/cav/KatzBDJK17}), but it is open how a well-defined extension of our
	logic looks like. Similarly, it would be useful to allow for more flexibility in the combination or
	aggregation functions of the considered GNNs. For example, allowing for
	arbitrary feedforward neural networks in the combination functions or considering the mean function
	as aggregation. Another idea is to allow for global readouts in the considered GNNs. A GNN with global
	readout does not only rely on local messages passing from neighbors to neighbors, and can instead consider the whole graph. Presumably, an
	extension of $\logicKsharpone$ would need something akin to a universal modality to captures such behaviour. 
	
	A second interesting direction would be to consider other classes of graphs. For instance, reflexive, transitive 
	graphs. Restricted types of graphs lead to different modal logics: $KT$ (validities on reflexive Kripke 
	models), $KD$ (on serial models), $S4$ (reflexive and transitive models), $KB$ (models where relations are 
	symmetric), $S5$ (models where relations are equivalence relations), etc.\ 
	\cite{DBLP:books/cu/BlackburnRV01}. The logic $\logicKsharpone$ defined in this paper is the counterpart 
	of modal logic $K$ with linear programs. In the same way, we could define $KT^\#$, $S4^\#$, $S5^\#$, etc. 
	For instance, $KB^\#$ would be the set of validities of $\logicKsharpone$-formulas over symmetric 
	models;  $KB^\#$ would be the logic used when GNNs are only used to recognize undirected pointed 
	graphs (for instance persons in a social network where friendship is undirected).
	
	A third direction of research would be to build a tool for solving verification and explainability issues of 
	GNNs using $\logicKsharpone$ as described in Section~\ref{section:overview}. The satisfiability problem of $\logicKsharpone$ is in PSPACE. Many practical problems are in PSPACE (model checking against a linear temporal logic property just to cite one \cite{DBLP:journals/jacm/SistlaC85}). This means that we may rely on heuristics to guide the search, but this needs thorough investigations.

\section*{Ethical Statement}
There are no ethical issues.

\section*{Acknowledgments}

This work was supported by the ANR EpiRL project ANR-22-CE23-0029, and by the
MUR (Italy) Department of Excellence 2023--2027. We would like to thank Stéphane Demri for discussions.

\bibliographystyle{named}
\bibliography{biblio}

\clearpage
	\appendix

\renewcommand{\thesection}{\hspace*{-1.0em}}
\renewcommand{\thesubsection}{\arabic{subsection}}
%\section{Appendix}
\section{Appendix}

This appendix contains the following technical material:
\begin{enumerate}
\item Main Definitions %\nameref{sec:appendix-definitions}
\item The Modality $\lboxupto{m}$ %\nameref{sec:modalityBoxm}
\item Proof that $\logicKsharp$ Can Express Properties That FO Can't (\Cref{example-expressivity-FO}) %\nameref{sec:FO-expressivity}
\item Complement of Proof of \Cref{thm:fromLogictoGNN} (Correctness of the Construction) % \nameref{sec:compl-proof-logic-to-GNN}
\item Proof of the Reduction From the Satisfiability Problem of $\logicKsharpone$ Into $\logicKsharpone_{tree}$-sat (\Cref{lemma-satKone-dag-tree}) % \nameref{proof-lemma-satKone-dag-tree}
\item Proof of the Reduction From the Satisfiability Problem of $\logicKsharpone$ Into $\logicKsharp$-sat (\Cref{lemma-satKoneK}) % \nameref{proof-lemma-satKoneK}
\item Relationship Between the Languages of $\logicKsharp$ and EML % \nameref{proof-relationship-ksharp-EML}
\end{enumerate}

\subsection{Main Definitions}\label{sec:appendix-definitions}
\setcounter{definition}{0}
A \emph{(labeled directed) graph} $G$ is a tuple $(\setvertices, \setedges, \labeling)$ such that $\setvertices$ is a finite set of vertices, $\setedges \subseteq \setvertices \times \setvertices$ a set of directed edges and $\labeling$ is a mapping from~$\setvertices$ to a valuation over a set of atomic propositions. We write  $\labeling(u)(p) = 1$ when atomic proposition $p$ is true in $u$, and $\labeling(u)(p) = 0$ otherwise. Given a graph $G$ and vertex $u \in \setvertices$, we call $(G,u)$ a \emph{pointed graph}.

Let $G = (\setvertices, \setedges, \labeling)$ be a graph. A \emph{state} $x$ is a mapping from $\setvertices$ into $\setR^\dimensionstate$ for some $\dimensionstate$. 
We use the term `state' for both denoting the map $x$ and also the vector $x(v)$ at a given vertex $v$. 
Suppose that the atomic propositions occurring in $G$ are $p_1, \dotsc, p_k$.
The \emph{initial state} $x_0$ is defined by: $$x_0(u) := (\labeling(u)(p_1), \dots, \labeling(u)(p_k)) \in \setR^d$$ for all $u \in V$.
An \emph{aggregation function} $\AGG$ is a function mapping finite multisets of vectors in $\mathbb{R}^{\dimensionstate}$ to vectors in $\mathbb{R}^{\dimensionstate}$.
A \emph{combination function} $\COMB$ is a function mapping a vector in $\mathbb{R}^{2\dimensionstate}$ to vectors in $\mathbb{R}^{d'}$.
A \emph{AC-GNN layer} $l$ of input dimension $d$ is defined by an aggregation function $\AGG$ and a combination function $\COMB$ of matching dimensions, meaning $\AGG$ expects and produces vectors from $\mathbb{R}^d$ and $\COMB$ has input dimensionality $2d$. 
The output dimension of $l$ is given by the output dimension of $\COMB$.
Then, a \emph{AC-GNN} is a tuple $(\layer^{(1)}, \dotsc, \layer^{(\nblayers)}, \CLS)$ where $\layer^{(1)}, \dotsc, \layer^{(\nblayers)}$ are 
$\nblayers$ AC-GNN layers and $\CLS : \setR^\dimensionstate \rightarrow \set{0, 1}$ is a \emph{classification function}. We assume that all GNNs are well-formed in the sense that
output dimension of layer $L^{(i)}$ matches input dimension of layer $L^{(i+1)}$ as well
as output dimension of $L^{(L)}$ matches input dimension of $\CLS$.
\begin{definition}%\label{def:our-GNNs}
Here we call a \emph{GNN} an AC-GNN $\aGNN$ where all aggregation functions are given by $\AGG (X) = \sum_{x \in X} x$, 
all combination functions are given by $\COMB(x,y) = \sigmabold(xC+yA +b)$ where $\sigmabold(x)$ is the componentwise application of the \emph{truncated ReLU} $\sigma(x) = max(0, min(1, x))$, where $C$ and $A$ are matrices of integer parameters and $b$ is a vector of integer parameters, and where the classification function is $\CLS(x) = \sum_i a_i x_i \geq 1$, with $a_i$ integers.
\end{definition}

Consider a countable set $\Ap$ of propositions. We define the language of logic $\logicKsharpone$ as the set of formulas generated by the following BNF:
\begin{align*}
	\phi & ::= p \mid \lnot \phi \mid \phi \lor \phi \mid \NTexpression \geq 0 \\ % \mid \NTexpression = 0 \\
	\NTexpression & ::= c \mid \istrue\phi \mid \modalitynumber \phi \mid \NTexpression + \NTexpression \mid c\times \NTexpression 
\end{align*}
where $p$ ranges over $\Ap$, and $c$ ranges over $\mathbb Z$.
The set of subformulas, $\subformulasof{\phi}$ is defined by induction on~$\phi$:
\begin{align*}
	\subformulasof p & = \set{p} \\
	\subformulasof {\lnot \phi} & = \set{\lnot \phi} \union \subformulasof{\phi} \\
	\subformulasof{\phi \lor \psi} & = \set{\phi \lor \psi} \union \subformulasof \phi \union \subformulasof \psi \\
	\subformulasof{\NTexpression \geq 0} & = \set{\NTexpression \geq 0} \union \bigcup \set{\subformulasof\psi \mid \text{$1_\psi$ or $\modalitynumber \psi$ in $\NTexpression$}}
\end{align*}
The modal depth of a formula, $\modaldepthof{\phi}$ and the modal depth of an expression, $\modaldepthof{\NTexpression}$ are defined by mutual induction on $\phi$ and $\NTexpression$:
	\begin{align*}
		\modaldepthof p, \modaldepthof{c} & = 0 \\
		\modaldepthof {\lnot \phi}, \modaldepthof {1_\phi} & = \modaldepthof{\phi} \\
		\modaldepthof {\phi \lor \psi} & = \max(\modaldepthof \phi,\modaldepthof \psi) \\
		\modaldepthof{\NTexpression \geq 0}, \modaldepthof {k\cdot \NTexpression}   &= \modaldepthof{\NTexpression} \\
		\modaldepthof {\modalitynumber \phi} & = \modaldepthof{\phi} + 1 \\
		\modaldepthof {\NTexpression_1 + \NTexpression_2} & = \max(\modaldepthof {\NTexpression_1},\modaldepthof {\NTexpression_2})
	\end{align*}

\begin{definition}
	A DAG of a formula is a graph in which nodes for $c$, $1_{\phi}$, $\modalitynumber \phi$, $\NTexpression + \NTexpression'$, $c \times \NTexpression$ have in-degree $1$.
\end{definition}

As in modal logic, a formula $\phi$ is evaluated in a pointed graph $(G, u)$ (also known as pointed Kripke model). 
We define the truth conditions $(G,u) \models \phi$ ($\phi$ is true in $u$) by 
	\begin{center}
		\begin{tabular}{lll}
			$(G,u) \models p$ & if & $\labeling(u)(p) = 1$, \\
			$(G,u) \models \neg \phi$ & if & it is not the case that $(G,u) \models \phi$, \\
			$(G,u) \models \phi \land \psi$ & if & $(G,u) \models \phi$ and $(G,u) \models \psi$, \\
			$(G,u) \models \NTexpression \geq 0$ & if &  $\semanticsvalue{\NTexpression}{G,u} \geq 0$, \\
			%$(G,u) \models \NTexpression = 0$ & if &  $\semanticsvalue{\NTexpression}{G,u} = 0$
		\end{tabular}
	\end{center}
	and the semantics $\semanticsvalue{\NTexpression}{G,u}$ (the value of $\NTexpression$ in $u$) of an expression $\NTexpression$ by mutual induction on $\phi$ and $\NTexpression$ as follows.
	\begin{center}
		$\begin{array}{ll}
			\semanticsvalue{c}{G, u} & = c, \\
			\semanticsvalue{\NTexpression_1+\NTexpression_2}{G, u} & = \semanticsvalue{\NTexpression_1}{G,u}+\semanticsvalue{\NTexpression_2}{G,u}, \\
			\semanticsvalue{c \times \NTexpression}{G, u} & = c \times \semanticsvalue{\NTexpression}{G,u}, \\
			\semanticsvalue{\istrue\phi}{G, u} & = \begin{cases}
				1 & \text{if $(G,u) \models \phi$} \\
				0 & \text{else},
			\end{cases} \\  
			\semanticsvalue{\modalitynumber\phi}{G, u} & = |\{v \in \setvertices \mid (u,v) \in \setedges \text{ and } (G,v) \models \phi\}|.
		\end{array}$
	\end{center}
We define $\semanticsof \phi$ as the set of the pointed graphs $G, u$ such that $G, u \models \phi$.
Furthermore, we say that $\phi$ is \emph{satisfiable} when there exists a pointed graph $G, u$ such that $G, u \models \phi$.
The \emph{satisfiability problem} of $\logicKsharpone$ (resp.\ $\logicKsharp$) is: given a DAG representation of a formula $\phi$ in the language of $\logicKsharpone$ (resp.\ $\logicKsharp$), is $\phi$ satisfiable?

\subsection{The Modality $\lboxupto{m}$}
\label{sec:modalityBoxm}

Remember that we can simulate the $\lbox$ operator with the formula $( \modalitynumber (\lnot \phi) = 0)$, or better, as $\lbox \phi :=( 0  \geq \modalitynumber (\lnot \phi) )$. We can also simulate an operator $\lboxupto{m} \phi$ meaning that the formula $\phi$ holds in all the vertices accessible from the root in $m$ steps or less.
\[
\lboxupto{m} \phi := \bigvee_{0\leq i \leq m} \underbrace{\Box \ldots \Box}_{i \text{ times}} \phi
\]
Observe that the size of a tree representation of $\lboxupto{m} \phi$ grows linearly wrt the size of $\phi$ and quadratically wrt $m$.
\begin{fact}\label{fact:sizeBoxm}
$|\lboxupto{m} \phi| = \mathcal{O}(m(m+|\phi|))$.
\end{fact}
Modality  $\lboxupto{m}$ enables us to reach the relevant parts of a model of a formula $\psi$, given $m$ is the modal depth of~$\psi$.

\subsection{Proof that $\logicKsharp$ Can Express Properties That FO Can't (\Cref{example-expressivity-FO})}\label{sec:FO-expressivity}

We show that % $\logicKsharp$ is more expressive than FO by proving that
the formula $\modalitynumber p \geq \modalitynumber q$ is not expressible by a FO formula $\phi(x)$. We observe that if the property ``for all vertices of a graph $\modalitynumber p \geq \modalitynumber q$'' is not expressible in FO, then the FO formula $\phi(x)$ doesn't exist, because if it existed the property would be expressible in FO by the formula $\forall x \phi(x)$.

For each integer $n > 0$, we consider the graphs $A_n$ and $B_n$ such that every vertices of $A_n$ verify $\modalitynumber p \geq \modalitynumber q$ while this is not the case for $B_n$ :

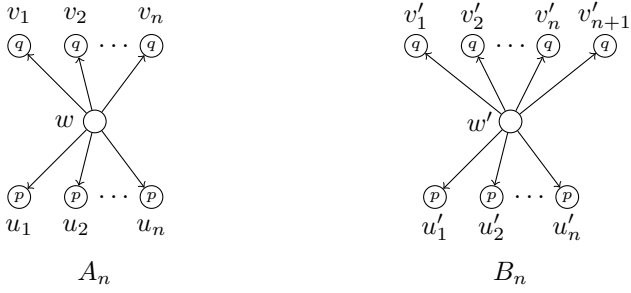
\begin{center}
	\begin{tikzpicture}
		\node[vertex] (w) at (-0.5, 1) {};
		\node[vertex] (u1) at (-1.5, 0) {$p$};
		\node[vertex] (u2) at (-0.75, 0) {$p$};
		\node at (-0.25, 0) {$\ldots$};
		\node at (-0.9, 1) {$w$};
		\node at (-1.5, -0.4) {$u_1$};
		\node at (-0.75, -0.4) {$u_2$};
		\node at (0.25, -0.4) {$u_n$};
		\node[vertex] (un) at (0.25, 0) {$p$};
		\node[vertex] (v1) at (-1.5, 2) {$q$};
		\node[vertex] (v2) at (-0.75, 2) {$q$};
		\node at (-0.25, 2) {$\ldots$};
		\node at (-1.5, 2.4) {$v_1$};
		\node at (-0.75, 2.4) {$v_2$};
		\node at (0.25, 2.4) {$v_n$};
		\node[vertex] (vn) at (0.25, 2) {$q$};
		\draw[->] (w) edge (v1);
		\draw[->] (w) edge (v2);
		\draw[->] (w) edge (vn);
		\draw[->] (w) edge (u1);
		\draw[->] (w) edge (u2);
		\draw[->] (w) edge (un);
		
		\node[vertex] (bw) at (5, 1) {};
		\node[vertex] (bu1) at (4, 0) {$p$};
		\node[vertex] (bu2) at (4.75, 0) {$p$};
		\node at (5.25, 0) {$\ldots$};
		\node at (4.6, 1) {$w'$};
		\node at (4, -0.4) {$u'_1$};
		\node at (4.75, -0.4) {$u'_2$};
		\node at (5.75, -0.4) {$u'_n$};
		\node[vertex] (bun) at (5.75, 0) {$p$};
		\node[vertex] (bv1) at (3.75, 2) {$q$};
		\node[vertex] (bv2) at (4.5, 2) {$q$};
		\node at (5, 2) {$\ldots$};
		\node at (3.75, 2.4) {$v'_1$};
		\node at (4.5, 2.4) {$v'_2$};
		\node at (5.5, 2.4) {$v'_n$};
		\node at (6.25, 2.4) {$v'_{n+1}$};
		\node[vertex] (bvn) at (5.5, 2) {$q$};
		\node[vertex] (bvn1) at (6.25, 2) {$q$};
		\draw[->] (bw) edge (bv1);
		\draw[->] (bw) edge (bv2);
		\draw[->] (bw) edge (bvn);
		\draw[->] (bw) edge (bvn1);
		\draw[->] (bw) edge (bu1);
		\draw[->] (bw) edge (bu2);
		\draw[->] (bw) edge (bun);
		
		\node at (-0.5, -1) {$A_n$};
		\node at (5, -1) {$B_n$};
	\end{tikzpicture}
\end{center}

An $n$-round Ehrenfeucht-Fra\"iss\'e game is a game between two players, the spoiler and the duplicator played on two graphs $A = (V_A,E_A,\ell_A)$ and $B = (V_B,E_B,\ell_B)$. We suppose that at each round the spoiler picks one graph and a vertex in this graph. The duplicator then chooses a vertex on the other graph. We have $n$ vertices ($a_1$,$a_2$,\ldots,$a_n$) chosen in $A$ and $n$ vertices ($b_1$,$b_2$,\ldots,$b_n$) chosen in $B$. The duplicator wins if and only if for $1 \leq i,j \leq n$ :

$a_i=a_j \Leftrightarrow b_i = b_j$

$(a_i,a_j) \in E_A \Leftrightarrow (b_i,b_j) \in E_B$

$\ell(a_i)= \ell(b_i)$

On the graphs $A_n$ and $B_n$, the duplicator wins the Ehrenfeucht-Fra\"iss\'e game with $n$ rounds: if the spoiler chooses $w$ (resp.\ $w'$) the duplicator chooses $w'$ (resp. $w$), if the spoiler chooses some $u_i$ or $v_i$ (resp.\ $u'_i$ or $v'_i$) the duplicator chooses $u'_j$ or $v'_j$ (resp.\ $u'_j$ or $v'_j$). (If the world chosen by spoiler has not been chosen in the previous round the duplicator pick a fresh index $j$. Otherwise, the duplicator picks the $j$ corresponding to the world chosen in the previous rounds). Since there are only $n$ distinct values that $i$ can take, the duplicator will win the game in $n$ rounds with her strategy. Thus the property ``for every vertices of a graph, $\modalitynumber p \geq \modalitynumber q$'' is not expressible in FO. % Therefore $\logicKsharp$ is more expressive than FO.

\subsection{Complement of Proof of \Cref{thm:fromLogictoGNN} (Correctness of the Construction)}
\label{sec:compl-proof-logic-to-GNN}

Let $l_i$ be the $i$th layer of $\aGNN_{\phi}$ and $\varphi_j$ such that $j \leq j$.
We show that for all vertices $v$ in some graph holds that $(x_i(v))_j = 1$ if $\phi_j$ is satisfied in $v$ and $(x_i(v))_j = 0$ if $\phi_j$ is not satisfied in $v$ . We call this property $(*)$.
Consider $l_1$. This implies that $\varphi_j = p_j$. We see that $C$, $A$ and $b$ are built such that they do not alter dimensions $1$ to $m$, namely those corresponding
to propositions. Therefore, property $(*)$ holds for $x_1$.  Now, assume that this property holds for $x_{i-1}$. Let $j < i$. Then, the semantics of $\varphi_j$ are already
represented in $x_{j-1}$. Furthermore, $\COMB_i$ preserves these representations and, thus, $\varphi_j$ is correctly represented in $x_i$. Now, consider the case
$j=i$. By assumption, the semantics of all subformulas of $\varphi_j$ are represented in $x_{i-1}$, meaning that the inputs $x$ and $y$ of $\COMB_i$ contain these
representations for $v$ respectively its neighbours. By design of $\COMB_i$, the formula $\varphi_j$ is then correctly evaluated. Therefore, $x_i$ fulfills property $(*)$.
Given that $(*)$ holds for all layers, we have that $x_n$ represents for each vertex $v$ and subformula $\varphi_i$ whether $v$ satisfies $\varphi_i$,
including $\varphi_n = \varphi$. In the end, $\CLS$ checks whether $x_n \geq 1$, which is equivalent to checking whether $\varphi$ is satisfied.

Next, we argue that we can build $\aGNN_{\phi}$ such that the number of layers $|\aGNN_{\phi}| \in 
\mathcal{O}(\modaldepthof{\varphi})$. 
However, we can use the same construction as above, but first we transform $\varphi$ 
into a formula $\varphi' \in \logicKsharpone$ by transforming formulas at the same modal level
into conjunctive normal form (CNF). In general, this leads to an exponential blow up of $\varphi'$. But,
now it is ensured that each such modal level has a nesting depth of boolean operators of at most $2$.
In other words, there exists an enumeration $(\phi_1,\ldots,\phi_{L})$ of subformulas of $\phi'$ such that 
$\phi_L = \phi'$ and with $L \in \mathcal{O}(\modaldepthof{\varphi})$. Now, the key insight is that on the GNN side we can build
a single layer in the manner of above, recognizing a single but arbitrary large disjunction or conjunction: 
assume that $\varphi_i = \bigvee_{j \in J_\textit{pos}} \psi_j \vee \bigvee_{j \in J_\textit{neg}} \neg \psi_{j} $. Then, we set 
$C_{ji} = 1$ for each $j \in J_\textit{pos}$ and  $C_{ji} = -1$ if $j \in J_\textit{neg}$ and $b_i = |J_\textit{neg}|$. Analogously,
assume that $\varphi_i = \bigwedge_{j \in J} \psi_j$. Then, we set $C_{ji} = 1$ for each $j \in J$ and $b_i = -|J| + 1$.
The correctness argument for this construction is exactly the same as above.

\subsection{Proof of the Reduction From the Satisfiability Problem of $\logicKsharpone$ Into $\logicKsharpone_{tree}$-sat (\Cref{lemma-satKone-dag-tree})}
\label{proof-lemma-satKone-dag-tree}

Let $\phi$ be a $\logicKsharpone$ formula, represented as a DAG $D_\phi = (N, \rightarrow)$.
(Arithmetic subexpressions may appear more than once. But formulas can, too, since there is no assumption of minimality.)
Given a node $n \in N$, we note $dag(n)$ the subDAG of $D_\phi$ with root $n$ generated by the transitive closure of $\rightarrow$.
Let $N_f \subseteq N$ be the set of nodes representing a formula (as opposed to an arithmetic expression), and let $\mu$ be the map from a node $n \in N_f$ to the subDAG $dag(n)$.
Let $<_f$ be a topological sorting of $N_f$.
\begin{example}\label{ex:DAG_N_map_topo}
	Consider the DAG depicted on \Cref{figure-dag}(a), representing a formula $\phi$.
	Let $N_f = \{1,2,3,4,5\}$, and let
	$\mu(1) \mapsto \phi$, $\mu(2)  \mapsto \modalitynumber (p \land q) \geq 1_{p \land q}$, $\mu(3) \mapsto p \land q$,
	$\mu(4)  \mapsto p$, and $\mu(5)  \mapsto q$.
	Let $<_f$ be the topological sorting of $N_f$ with $<_f = 1,2,3,4,5$.
\end{example}

We are going to (i)~reduce $D_\phi$ to a single node containing a propositional variable representing $\phi$, and (ii)~build a tree $\Psi$ to store the semantics of the original nodes in $D_\phi$. Let $\Psi = \top$, thus a tree with one node containing $\top$.
Iterating over the elements $n$ from the end to the start of $<_f$:
\begin{enumerate}
	\item Introduce a fresh propositional variable $p_n$, represented as a DAG/tree with only one node.
	\item Let $\Psi := \Psi \land (p_n \leftrightarrow cp(\mu(n)))$, where $cp(\mu(n))$ is a copy of the DAG mapped by $\mu(n)$.
	\item Replace $n$ in $D_\phi$ with the single node containing $p_n$ and remove the edges going out from it.
\end{enumerate}
At the end, define the tree\footnote{Notice that in \Cref{section:complexitylogic}, for explicit clarity, we suggest to use $\lboxupto{m}$ for every conjunct of $\Psi$. Here, we use the (shorter) equivalent formula $\lboxupto{m} \Psi$ instead. This has no impact on the correctness of the arguments since $\lboxupto{m} (\psi_1 \land \psi_2) \leftrightarrow \lboxupto{m} \psi_1 \land  \lboxupto{m}\psi_2$ is a theorem of all standard modal logics.
} \[\phi_t = p_n \land \lboxupto{m} \Psi . \]

\begin{example}
Continuing \Cref{ex:DAG_N_map_topo},
we introduce the propositional variables $p_5$, $p_4$, $p_3$, $p_2$, and $p_1$.
We define the tree 
\begin{align*}
\Psi = & (p_5 \leftrightarrow q) \land
 (p_4 \leftrightarrow p) \land
 (p_3 \leftrightarrow p_4 \land p_5) \land\\
& (p_2 \leftrightarrow \modalitynumber p_3 \leq 1_{p_3}) \land
 (p_1 \leftrightarrow p_3 \lor p_2)
\end{align*}
and finally
$\phi_t = p_1 \land \lboxupto{1}\Psi$.
\end{example}

At every iteration, $\Psi$ can be represented by the DAG/tree:

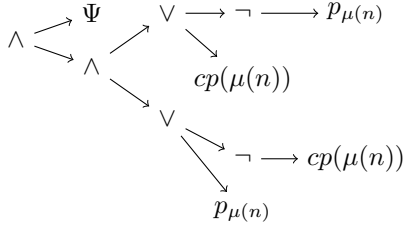
\begin{tikzpicture}[yscale=0.7, baseline=0mm]
		\node (1) at (0, 0) {$\land$};
		\node (2) at (1, 0.5) {$\Psi$};
		\node (3) at (1, -0.5) {$\land$};

		\node (31) at (2, 0.5) {$\lor$};
		\node (32) at (2, -1.5) {$\lor$};

                \node (311) at (3, 0.5) {$\lnot$};
                \node (312) at (3, -0.75) {$cp(\mu(n))$};
                \node (3111) at (4.5, 0.5) {$p_{\mu(n)}$};

                \node (321) at (3, -2.25) {$\lnot$};
		\node (322) at (3, -3.25) {$p_{\mu(n)}$};
                \node (3211) at (4.5, -2.25) {$cp(\mu(n))$};

		\draw[->] (1) -- (2);
		\draw[->] (1) -- (3);
		\draw[->] (3) -- (31);
		\draw[->] (3) -- (32);

		\draw[->] (31) -- (311);
		\draw[->] (31) -- (312);
		\draw[->] (311) -- (3111);

		\draw[->] (32) -- (321);
		\draw[->] (32) -- (322);
		\draw[->] (321) -- (3211);
                \end{tikzpicture}

So, at every iteration, $|\Psi| \leq \mathcal{O}(|\Psi| + 2|\phi|)$.
The number of nodes in the DAG representation of $\phi$ cannot be larger than $|\phi|$.
Hence, the size of $\Psi$ at the end will be at most $\mathcal{O}(2|\phi|^2)$.
Considering \Cref{fact:sizeBoxm}, the size of $\phi_t$ is thus at most $\mathcal{O}(\modaldepthof{\phi}(\modaldepthof{\phi}+2|\phi|^2)) \leq
\mathcal{O}(2|\phi|^3)$.

By construction and by the semantics of $\logicKsharpone$, we have that
at every iteration, $\phi$ is satisfiable iff  $D_\phi \land \lboxupto{m}\Psi$ is satisfiable. So
$\phi_t$ is satisfiable iff $\phi$ is satisfiable.

\subsection{Proof of the Reduction From the Satisfiability Problem of $\logicKsharpone$ Into $\logicKsharp$-sat (\Cref{lemma-satKoneK})}
\label{proof-lemma-satKoneK}

We will show that any $\logicKsharpone$ formula (that is not a $\logicKsharp$ formula), can be rewritten in poly-time into an equi-satisfiable $\logicKsharpone$ formula containing one less $1_\chi$ arithmetic subexpression. This rewriting can then be iterated until we obtain a pure $\logicKsharp$ formula.

\newcommand\pindex{k}
\newcommand\chipindex{{\chi_\pindex}}
\newcommand\taupindex{{\tau_\pindex}}

Let $\phi_\pindex$ be a $\logicKsharpone$ formula, and $1_\chipindex$ be an arithmetic subexpression of $\phi_\pindex$ such that $\chipindex$ is a $\logicKsharp$ formula.
In order to get rid of the arithmetic subexpression $E_\pindex = 1_\chipindex$ in $\phi_\pindex$, we simulate the truth of $\chipindex$ by the presence of a unique \emph{impure} successor labeled by the fresh atom $p_\chipindex$, and the falsity of $\chipindex$ by the absence of an \emph{impure} successor labeled by the atom $p_\chipindex$. Henceforth, the value of $1_\chipindex$ can be captured by the value of $\modalitynumber p_\chipindex$.

The presence/absence of a unique impure successor is captured by the formula $\zeta^\modalitynumber_\pindex$ defined next, in which $m$ is the modal depth of $\phi$. 
\[
\zeta^\modalitynumber_\pindex  := \lboxupto{m} \left (
(\chipindex \limply (\modalitynumber p_\chipindex = 1) ) \land (\lnot \chipindex \limply (\modalitynumber p_\chipindex = 0))
\right )
\]
Now we define $\taupindex(\phi)$ by induction:

\begin{align*}
	\taupindex(p) & = p \\
	\taupindex(\lnot \phi) & = \lnot \taupindex(\phi) \\
	\taupindex(\phi \lor \psi) & = \taupindex(\phi) \lor \taupindex(\psi) \\
	\taupindex(\NTexpression \geq 0) & = \taupindex(\NTexpression) \geq 0 \\
	\taupindex(\NTexpression+\NTexpression') & = \taupindex(\NTexpression) + \taupindex(\NTexpression') \\
	\taupindex(c \times \NTexpression) & = c \times \taupindex(\NTexpression) \\    
	\taupindex(1_\phi) & = \begin{cases}
		\modalitynumber p_\chipindex & \text{when } \phi = \chipindex \\
		1_{\taupindex(\phi)} & \text{otherwise}
	\end{cases}\\
	\taupindex(\modalitynumber \phi) & = \modalitynumber (\taupindex(\phi) \land \lnot p_\chipindex)
\end{align*}
Finally we define $tr(\phi_\pindex) := \taupindex(\phi_\pindex) \land \zeta^\modalitynumber_\pindex$.

The size of $tr(\phi_\pindex)$ is only polynomially larger than $\phi_\pindex$. Indeed,  $|\taupindex(\phi_\pindex)| \leq \mathcal{O}(|\phi_k|))$ and (considering \Cref{fact:sizeBoxm} again) $|\zeta^\modalitynumber_\pindex| = \mathcal{O}(\modaldepthof{\phi_k}(\modaldepthof{\phi_k} + 2|\chi_k|)) \leq \mathcal{O}(3|\phi_k|^2)$.

Let us prove that $\phi_\pindex$ is satisfiable iff $tr(\phi_\pindex)$ is satisfiable.

\noindent \emph{Left to right.}~
Suppose $\phi_\pindex$ is satisfiable. There is a pointed graph $G, u$ such that $G, u \models \phi_\pindex$, with $G = (\setvertices, \setedges, \labeling)$.
We define $G' = (\setvertices', \setedges', \labeling')$, where:
\begin{itemize}
	\item $\setvertices' = \setvertices \cup \{ v_{impure} \mid G, v \models \chipindex \} $
	\item $\setedges' = \setedges \cup \{ (v,v_{impure}) \mid G, v \models \chipindex \ \}$
	\item $\labeling'(v)(p_\chipindex) = \begin{cases}
		1 & \text{if } v \not\in \setvertices \text{ ($v$ is impure)}\\
		0 & \text{otherwise}
	\end{cases}$
	\item $\labeling'(v)(p) = \labeling'(v)(p)$ when $p \not = p_\chipindex$.
\end{itemize}

By construction, we have $G', u \models \zeta^\modalitynumber_\pindex$. It remains to prove that $G', u \models \taupindex(\phi_\pindex)$. We proceed by induction on the complexity of formulas and arithmetic expressions. The base cases of $p$ and $c$ are immediate. We show the non-trivial cases: arithmetic expressions $\modalitynumber \phi$ and $1_\phi$ for some formula $\phi$. Suppose towards induction that $G, v \models \phi$ iff $G', v \models \taupindex(\phi)$ for every $v \in \setvertices$.
We have $\semanticsvalue{\modalitynumber \phi}{G,u}$
\[
\begin{array}{ll}
	%\semanticsvalue{\modalitynumber \phi}{G,u}
	& =
	|\{v \in \setvertices \mid (u,v) \in \setedges \text{ and } G, v \models \phi \}| \\
	& =
	|\{v \in \setvertices' \mid (u,v) \in \setedges \text{ and } G', v \models \taupindex(\phi) \}| \hfill \text{ (by i.h.)} \\
	& =
	|\{v \in \setvertices' \mid (u,v) \in \setedges' \text{ and } G', v \models \lnot p_\chipindex\\ & \hfill \text{ and } G', v \models \taupindex(\phi) \}| \\
	& =
	\semanticsvalue{\modalitynumber (\taupindex(\phi) \land \lnot p_\chipindex)}{G',u}\\
	& =
	\semanticsvalue{\taupindex(\modalitynumber\phi)}{G',u}\\
\end{array}
\]
When $\phi \not = \chipindex$, $\semanticsvalue{1_\phi}{G,u} = \semanticsvalue{\taupindex(1_\phi)}{G,u}$ follows directly from the i.h. When $\phi = \chipindex$, we have:
\[
\begin{array}{ll}
	\semanticsvalue{1_\phi}{G,u}
	& = \begin{cases}
		1 & \text{if } G, u \models \chipindex\\
		0 & \text{otherwise.}
	\end{cases} \hfill \text{ (by definition)} \\
	& = \semanticsvalue{\modalitynumber p_\chipindex}{G',u} \hfill \text{ (by construction)}\\
	& = \semanticsvalue{\taupindex(1_\chipindex)}{G',u}
\end{array}
\]
So, $G', u \models \taupindex(\phi_\pindex)$.
So $G', u \models \taupindex(\phi_\pindex) \land \zeta^\modalitynumber_\pindex$.
So $tr(\phi_\pindex)$ is satisfiable.

\medskip
\noindent \emph{Right to left.}~
Suppose $tr(\phi_\pindex)$ is satisfiable.
There is a pointed graph $G', u$ such that $G', u \models tr(\phi_\pindex)$, with $G' = (\setvertices', \setedges', \labeling')$.
We define $G = (\setvertices, \setedges, \labeling)$, where:
\begin{itemize}
	\item $\setedges = \setedges' \setminus \{ (v,w) \in \setedges' \mid G', w \models p_\chipindex \}$
	\item $\setvertices = \setvertices'$
	\item $\labeling(v)(p) = \labeling'(v)(p)$, for $p \not = p_\chipindex$
\end{itemize}
Since $G', u \models tr(\phi_\pindex)$, we have in particular that $G', u \models \zeta^\modalitynumber_\pindex$. It means that every state in $G'$ reachable from $u$ in no more steps than the modal depth of $\phi_\pindex$ has one (and exactly one) $p_\chipindex$-successor when $\chipindex$ is true in that state, and none otherwise.

We must show that $G, u \models \phi_\pindex$. We proceed by induction.
One base case is $\modalitynumber p_\chipindex$. We have that $\semanticsvalue{\modalitynumber p_\chipindex}{G',u}$ is $1$ if $G', u \models \chipindex$ and $0$ otherwise; which is $\semanticsvalue{1_\chipindex}{G',u} = \semanticsvalue{1_\chipindex}{G,u}$.
Suppose towards induction that $G', v \models \taupindex(\phi)$ iff $G, v \models \phi$ for every $v \in \setvertices$.
The only remaining non-trivial case is $\modalitynumber (\taupindex(\phi) \land \lnot p_\chipindex)$. We have: $\semanticsvalue{\modalitynumber (\taupindex(\phi) \land \lnot p_\chipindex)}{G',u}$

\[
\begin{array}{ll}
	%\semanticsvalue{\modalitynumber (\taupindex(\phi) \land \lnot p_\chipindex)}{G',u}
	& = | \{ v \in \setvertices' \mid (u,v) \in \setedges' \text{ and } G', v \models \taupindex(\phi) \land \lnot p_\chipindex \} |\\
	& = | \{ v \in \setvertices' \mid (u,v) \in \setedges' \text{ and } G', v \models \lnot p_\chipindex \\ & \hfill \text{ and } G', v \models \taupindex(\phi) \} |\\
	& = | \{ v \in \setvertices' \mid (u,v) \in \setedges' \text{ and } G', v \models \lnot p_\chipindex  \\ & \hfill \text{ and } G, v \models \phi \} | \hfill \text{ (by i.h.)} \\
	& = | \{ v \in \setvertices \mid (u,v) \in \setedges \text{ and } G, v \models \phi \} |\\
	& = \semanticsvalue{\modalitynumber \phi}{G,u}
	
\end{array}
\]

So $G, u \models \phi_\pindex$. Hence $\phi_\pindex$ is satisfiable.

\medskip

To finish the proof, it suffices to iterate and use the satisfiability preservation of $tr$.
Let $\phi_0$ be an arbitrary $\logicKsharpone$ formula.
Let $i = 0$.
While $\phi_i$ contains $1_{\chi_i}$ subformulas, (1)~pick one such that $\chi_i$ is a $\logicKsharp$ formula, (2)~compute $\zeta^\modalitynumber_i$, (3)~compute $\phi_{i+1} := \tau_i(\phi_i)$, (4)~increment $i$.
At the end, the formula $\phi_i \land \bigwedge_{0 \leq j < i} \zeta^\modalitynumber_j$ is satisfiable iff $\phi_0$ is satisfiable, and only polynomially larger than $\phi_0$.

\subsection{Relationship Between the Languages of $\logicKsharp$
%$\logicKsharp_{(tree)}$
and EML}
\label{proof-relationship-ksharp-EML}

The logic $\logicKsharp$ is not exactly a `fragment' of EML~\cite{DBLP:journals/japll/DemriL10}. In particular, the language of EML does not give direct access to integer constants $c$ in arithmetic expressions. But it is not a limitation, and the transformation is rather trivial.

W.l.o.g.\ and without exponential blowup in size, $\logicKsharp$ subformulas $\NTexpression \geq 0$ are of the form $\sum_i c_i \times \modalitynumber{\phi_{i}} + c \geq 0$, and $c_i \not = 0$ for every $i$. 
The $\logicKsharp$ formula $\sum_i c_i \times \modalitynumber{\phi_{i}} + c \geq 0$ is then transformed as follows:
\begin{itemize}
\item if $c \geq 0$: we have $\NTexpression \geq 0$ iff $\sum - c_i \times \modalitynumber{\phi_{i}} - c \leq 0$ iff $\sum - c_i \times \modalitynumber{\phi_{i}}  \leq c$ iff $\lnot (\sum - c_i \times \modalitynumber{\phi_{i}}  > c)$ which is an EML formula.
\item if $c < 0$: we have $\NTexpression \geq 0$ iff $\sum c_i \times \modalitynumber{\phi_{i}} \leq - c$ iff $\lnot (\sum c_i \times \modalitynumber{\phi_{i}} > - c)$ which is an EML formula.
\end{itemize}

\end{document}